\documentclass{article}

\usepackage[preprint]{arxiv}
\usepackage{tcolorbox}
\usepackage[utf8]{inputenc} 
\usepackage[T1]{fontenc}    
\usepackage[hidelinks]{hyperref}
\usepackage{url}            
\usepackage{booktabs}       
\usepackage{amsfonts}       
\usepackage{nicefrac}       
\usepackage{microtype}      
\usepackage{xcolor}         
\usepackage{graphicx}
\usepackage{subcaption}
\usepackage{soul}
\usepackage{fontawesome5}
\definecolor{gain}{RGB}{0,90,180}
\definecolor{A}{RGB}{0,176,80}
\definecolor{B}{RGB}{0,176,240}
\definecolor{prior}{RGB}{219, 166, 1}
\definecolor{gradient}{RGB}{156, 65, 225}

\usepackage{mathtools}
\usepackage{amsthm}
\usepackage{aliascnt}

\usepackage{amsmath,amsfonts,amsthm,amssymb,bm,bbm}

\def\eqref#1{equation~\ref{#1}}

\def\1{\bm{1}}

\DeclareMathAlphabet{\mathsfit}{\encodingdefault}{\sfdefault}{m}{sl}
\SetMathAlphabet{\mathsfit}{bold}{\encodingdefault}{\sfdefault}{bx}{n}

\def\gA{{\mathcal{A}}}
\def\gB{{\mathcal{B}}}

\def\gD{{\mathcal{D}}}

\def\gJ{{\mathcal{J}}}

\def\gS{{\mathcal{S}}}

\def\gV{{\mathcal{V}}}

\def\sN{{\mathbb{N}}}

\def\sR{{\mathbb{R}}}

\newcommand{\E}{\mathbb{E}}

\newcommand{\Sp}[1]{\left(#1\right)}
\newcommand{\Mp}[1]{\left[#1\right]}
\newcommand{\Bp}[1]{\left\{#1\right\}}

\newcommand{\elbo}{\text{ELBO}}
\newcommand{\eubo}{\text{EUBO}}
\usepackage{wrapfig}
\usepackage{algorithm}
\usepackage{algorithmic}
\usepackage{placeins}

\usepackage[capitalize,noabbrev]{cleveref}

\theoremstyle{plain}
\newtheorem{theorem}{Theorem}[section]
\newaliascnt{proposition}{theorem}
\newtheorem{proposition}[proposition]{Proposition}
\aliascntresetthe{proposition}
\newaliascnt{lemma}{theorem}
\newtheorem{lemma}[lemma]{Lemma}
\aliascntresetthe{lemma}
\newaliascnt{corollary}{theorem}

\aliascntresetthe{corollary}
\theoremstyle{definition}
\newaliascnt{definition}{theorem}
\newtheorem{definition}[definition]{Definition}
\aliascntresetthe{definition}
\newaliascnt{assumption}{theorem}
\newtheorem{assumption}[assumption]{Assumption}
\aliascntresetthe{assumption}
\theoremstyle{remark}
\newaliascnt{remark}{theorem}
\newtheorem{remark}[remark]{Remark}
\aliascntresetthe{remark}
\crefname{definition}{definition}{definitions}
\Crefname{definition}{Definition}{Definitions}
\crefname{assumption}{assumption}{assumptions}
\Crefname{assumption}{Assumption}{Assumptions}
\crefname{proposition}{proposition}{propositions}
\Crefname{proposition}{Proposition}{Propositions}
\crefname{lemma}{lemma}{lemmas}
\Crefname{lemma}{Lemma}{Lemmas}
\crefname{remark}{remark}{remarks}
\Crefname{remark}{Remark}{Remarks}

\usepackage[textsize=tiny]{todonotes}

\title{DiPOD: Diffusion Policy Optimization without Drifting Apart}

\newcommand{\affilmark}[1]{\textsuperscript{#1}}
\newcommand{\equaladvmark}{\textsuperscript{\textdaggerdbl}}
\newcommand{\authorcell}[2]{%
  \makebox[0.30\textwidth][c]{%
    \begin{tabular}[t]{@{}c@{}}
      #1\\[-0.2ex]
      {\normalfont\footnotesize\texttt{#2}}
    \end{tabular}%
  }%
}

\author{%
  \begin{tabular}{@{}ccc@{}}
    \authorcell{Haozhe Jiang\affilmark{1,2}}{ericjiang@berkeley.edu} &
    \authorcell{Haiwen Feng\affilmark{1,2}}{haiwen.feng@berkeley.edu} &
    \authorcell{Pieter Abbeel\affilmark{1}}{pabbeel@cs.berkeley.edu} \\[1.0ex]
    \authorcell{Jiantao Jiao\affilmark{1,3}}{jiantao@cs.berkeley.edu} &
    \authorcell{Angjoo Kanazawa\affilmark{1}\equaladvmark}{kanazawa@eecs.berkeley.edu} &
    \authorcell{Nika Haghtalab\affilmark{1}\equaladvmark}{nika@berkeley.edu}
  \end{tabular}
}

\begin{document}

\maketitle
\begingroup
\renewcommand{\thefootnote}{\fnsymbol{footnote}}
\footnotetext[0]{Affiliations: \textsuperscript{1} UC Berkeley; \textsuperscript{2} Impossible, Inc.; \textsuperscript{3} NVIDIA.}
\footnotetext[3]{Equal advising.}
\endgroup

\vspace{-0.9em}
\begin{center}
{\normalfont\large
  \begin{tabular}{@{}c@{}}
    \faGlobe\; Paper website: \url{https://astro-eric.github.io/blogs/dipod/}\\[0.35ex]
    \faGithub\; Code: \href{https://github.com/Astro-Eric/DiPOD-release}{Astro-Eric/DiPOD-release}
  \end{tabular}%
}
\end{center}
\vspace{-0.1em}

\begin{abstract}
    RL post-training has become increasingly pivotal for improving diffusion policies, but existing diffusion policy-gradient methods are often unstable and cannot achieve reliable policy improvement. We identify the cause as the double-drift phenomenon: optimizing a variational surrogate can let the ELBO separate from the true log-likelihood, which then makes the resulting proxy policy gradient misaligned with the true policy gradient of expected return. We propose \textbf{DiPOD}, a diffusion policy optimization framework that maintains tight-bound behavior throughout training by interleaving self-distillation with policy-improving gradient updates. This leads to a simple and practical algorithm: augmenting each diffusion policy-gradient update with an on-policy ELBO regularizer. Across diffusion language model post-training and continuous-control diffusion policies, DiPOD substantially stabilizes training and reaches higher rewards than previous methods.
\end{abstract}

\section{Introduction}
\begin{figure}
  \centering
  \includegraphics[width=1\textwidth]{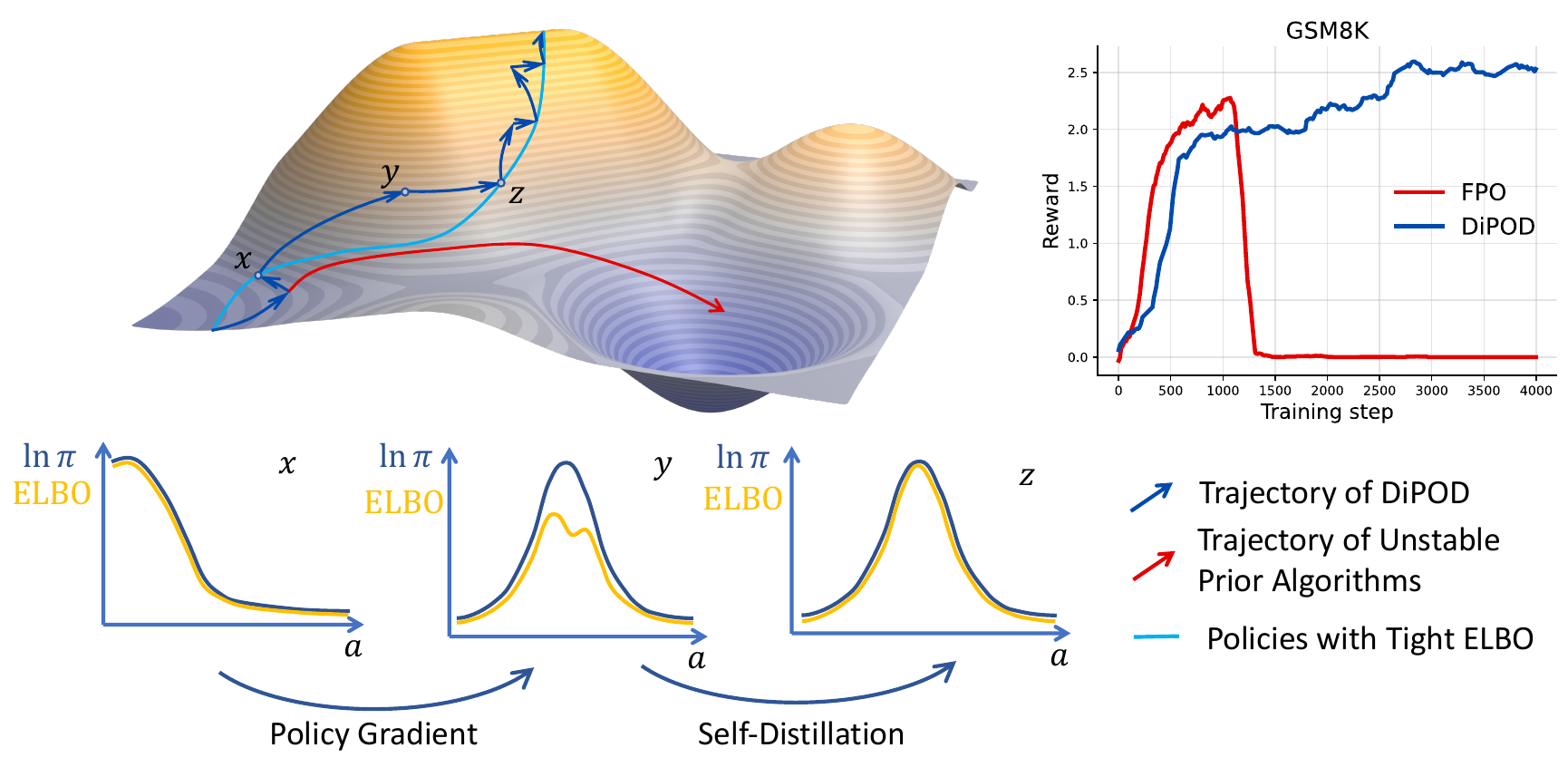}
  \caption{
\textbf{Illustration of DiPOD} using an expected-return landscape over the parameter space of the policy. We want the policy to both \textcolor{A}{guarantee policy improvement} and \textcolor{B}{have a tight ELBO}, as illustrated by the colored curves in the figure. Prior proxy-based algorithms (the plotted reward curve comes from a real FPO~\citep{mcallister2025flow} experiment on GSM8K~\citep{cobbe2021training}; see \Cref{fig:reward_dynamics}) can initially achieve policy improvements, but later become unstable as the ELBOs are no longer tight. In contrast, DiPOD interleaves self-distillation steps and adequate gradient updates, as illustrated by the blue arrows. Gradient updates guarantee policy improvements locally, and self-distillation brings the parameter to where the ELBO is tight while keeping the expected reward unchanged.
  }
  \label{fig:teaser}
\end{figure}

Diffusion models are emerging as a powerful paradigm for discrete generation in language, code, and mathematical reasoning, while diffusion and flow models more broadly support strong performance in continuous domains such as images~\citep{ho2020denoising,song2020denoising,song2019generative}, videos~\citep{ho2022video,ho2022imagen}, and robotic control~\citep{black2410pi0,bjorck2025gr00t}. This promise is especially compelling for reasoning with diffusion LLMs (dLLMs)~\citep{nie2025large,xie2025dream,song2025seed}, which offer fast parallel sampling and flexible non-autoregressive decoding~\citep{jiang2025diffusion,ermon2026mercury2,chen2026dflash}. Yet post-training diffusion policies with reinforcement learning remains fundamentally difficult: standard policy-gradient methods rely on the log-likelihood $\log \pi_\theta(a|o)$, while for diffusion models this quantity is generally not tractable. Recent methods~\citep{zhao2025d1,mcallister2025flow,wang2025spg} have therefore introduced a range of more tractable proxies for the likelihood, but these approaches still do not provide a systematic understanding of when proxy-based policy-gradient updates remain reliable. This gap is particularly consequential for dLLMs, where unstable RL post-training causes reasoning gains to lag far behind those of autoregressive counterparts.

\begin{tcolorbox}[colback=black!2,colframe=black!20,boxrule=0.4pt,arc=2pt,
  left=4pt,right=4pt,top=4pt,bottom=2pt]
In this work, we propose \textbf{DiPOD}---\textbf{Di}ffusion \textbf{P}olicy optimization with\textbf{O}ut \textbf{D}rifting apart---a principled framework for \textbf{\ul{reliable policy-gradient updates}} that is \textbf{\ul{native to diffusion and flow policies}}. When applied to dLLM post-training, DiPOD substantially stabilizes learning and improves reasoning performance on benchmarks such as GSM8K, MATH500, Countdown, and Sudoku, including being the first method to saturate Sudoku in the zero-shot setting (\Cref{sec:exp_language,tab:reasoning_aggregate}).
\end{tcolorbox}

As an RL post-training method native to diffusion policies, DiPOD is designed to satisfy two desiderata. First, \textcolor{A}{(A)} \emph{policy-gradient updates should improve the downstream objective} as this is the main goal of post-training. However, reward improvement alone is not enough for improving diffusion models, especially when the probabilistic structure that makes the policy a coherent diffusion model is destroyed in the process.
Hence, in addition, we require \textcolor{B}{(B)} \emph{fine-tuning to preserve the diffusion-model structure by keeping the proxy objective tight}. DiPOD meets both criteria. As illustrated in \Cref{fig:teaser}, DiPOD prevents proxy updates from drifting away from the intended policy-improvement direction by \emph{interleaving} (i) a policy-preserving self-distillation step that ensures minimal discrepancy between the likelihood and its proxies, with (ii) \emph{adequate} policy-gradient updates\footnote{This is a technical term (see Definition~\ref{def:adequate}) for a common type of policy-gradient update that is accurate in the absence of large discrepancy between likelihood and its proxies.} that improve expected return. From this framework, we derive a surprisingly simple drop-in implementation: add a per-update ELBO regularization term to the diffusion policy-gradient update (\Cref{alg:implementation}).

\paragraph{Variational-Inference Approaches to RL Post-training.}
Why is updating diffusion with policy gradient hard in the first place? Treating the denoising chain as an MDP makes likelihoods tractable~\citep{black2023training,ren2024diffusion}, but it ties training to a sampler and is a poor fit for dLLMs, whose post-training should preserve flexible decoding orders and inference budgets. Recent \emph{non-MDP} dLLM methods therefore approximate likelihoods directly, using mean-field or one-step estimators, partial dependency restorations, or imposed autoregressive orders~\citep{zhao2025d1,xie2025dream}; these can alter the optimized policy from the executed diffusion policy. Variational-inference (VI) approaches instead replace $\log \pi_\theta$ with evidence bounds such as ELBO~\citep{mcallister2025flow} or EUBO~\citep{wang2025spg}, preserving diffusion's native sampling flexibility while connecting naturally to pretraining objectives.

The key caveat is that this faithfulness is only local. Variational diffusion-RL methods can work well \emph{initially}: near a well-pretrained initialization where the evidence bound is tight, ELBO remains a good local proxy for the true log-likelihood. However, RL updates do not preserve this tight-bound regime. As optimization proceeds, the discrepancy between ELBO and log-likelihood can grow, and our analysis pinpoints the Achilles heel: \textcolor{B}{once ELBO drifts from log-likelihood,} \textcolor{A}{policy-gradient fidelity drifts with it.} We refer to this coupled effect as the \textbf{double drift} phenomenon.

DiPOD addresses this by making tight-bound behavior an explicit design constraint. Our key observation is that many VI-based estimators are \emph{adequate} (Definition~\ref{def:adequate}): if the diffusion model were perfectly trained so that the evidence bound is tight, the proxy gradient matches the true log-likelihood gradient, and a policy-gradient step based on the proxy coincides with the true policy gradient.
DiPOD makes use of this by repeatedly \emph{pulling the model back} toward a tight-bound regime via self-distillation---which tightens the evidence bound under a reference rollout distribution without changing the policy distribution---and only then taking policy-gradient steps. In the idealized interleaved form shown in \Cref{fig:teaser}, gradient updates guarantee local policy improvement, while self-distillation restores tightness without changing the policy output distribution.

Practically, we implement a simple approximation of this interleaving: for each batch of rollouts, we update parameters using the usual diffusion policy gradient estimator \emph{plus} an ELBO maximization term on the same rollout data. This additional ELBO term acts as a regularizer that reduces variational discrepancy on-policy, improving alignment between the proxy gradient and the true policy gradient. Empirically, this simple DiPOD implementation is surprisingly effective: it substantially stabilizes training and improves reward optimization across both discrete dLLM post-training tasks (e.g., GSM8K~\cite{cobbe2021training}, MATH500~\cite{lightman2023let}, Sudoku~\cite{arelSudokuGenerator}, Countdown~\cite{pan2025tinyzero}) and continuous-control flow policies (e.g., motion tracking for humanoid robots; details in~\Cref{sec:exp}).

In summary, we make three contributions:
\begin{itemize}
    \item We identify a fundamental failure mode in variational diffusion RL: \textbf{double drift}. As RL updates loosen the evidence bound, \textcolor{B}{ELBO drifts from log-likelihood}, which in turn causes \textcolor{A}{proxy policy-gradient updates to drift from the intended policy-gradient direction} (\Cref{sec:limitation}).
    \item We propose \textbf{DiPOD}, a principled framework for reliable diffusion policy optimization that makes staying in the tight-bound regime an explicit design objective, achieved by interleaving policy-gradient updates with policy-preserving self-distillation (\Cref{sec:self_distillation}).
    \item We derive a \textbf{simple, drop-in practical algorithm} that instantiates this principle by adding a per-update ELBO regularizer, substantially improving the stability and reward optimization of diffusion RL in practice (\Cref{alg:implementation,sec:exp}).
\end{itemize}

\subsection{Related Work}
\paragraph{Diffusion models.} Diffusion models are a powerful class of generative models, with applications in image generation~\citep{ho2020denoising,song2020denoising,song2019generative}, video generation~\citep{ho2022video,ho2022imagen}, and robotics~\citep{black2410pi0,bjorck2025gr00t}. More recently, diffusion language models have emerged as a promising alternative to autoregressive models for fast and flexible decoding~\citep{nie2025large,xie2025dream,song2025seed}. We study diffusion policies through the variational-inference perspective~\citep{kingma2023understanding,sohl2015deep}, which is the lens used by DiPOD to reason about likelihood proxies and evidence bounds.

\paragraph{Policy gradients.} Policy-gradient methods directly optimize a parameterized policy for expected return~\citep{williams1992simple,schulman2017proximal}, and have been central to high-dimensional control~\citep{schulman2016high}, locomotion~\citep{rudin2022learning}, manipulation~\citep{schwarke2023curiosity}, and modern language-model post-training~\citep{shao2024deepseekmath}. DiPOD targets this online policy-optimization setting for diffusion policies.

\paragraph{RL with diffusion policies.} The main obstacle to applying policy gradients to diffusion policies is that the exact log-likelihood is intractable. Prior work either treats the denoising process itself as an MDP~\citep{black2023training,ren2024diffusion}, replaces likelihoods with variational evidence bounds~\citep{mcallister2025flow,wang2025spg}, or simplifies dependencies in diffusion language models to obtain tractable likelihood surrogates~\citep{zhao2025d1,yang2025mmada,xie2025dream,tang2025wd1}. These approaches trade off tractability, sampler flexibility, and alignment between the optimized proxy and the executed diffusion policy. DiPOD focuses on the variational line and stabilizes it by keeping the evidence bound aligned with the true likelihood during policy optimization. A fuller discussion of related work appears in \Cref{app:related_work}.

\section{Preliminaries}\label{sec:prelim}

\paragraph{Policy Gradient Method.} The goal of reinforcement learning (RL) is to learn a policy $\pi_\theta$ that maximizes the expected return in an environment. Here $\theta$ denotes the parameters of the policy. At timestep $t$, the policy takes an observation $o_t$ from the environment, then executes action $a_t\sim\pi_\theta(\cdot|o_t)$, and the environment provides reward $r_t$ as feedback. The return is defined as the cumulative reward until the environment reaches a terminal state. A policy gradient algorithm updates the policy parameters $\theta$ using an estimator of the gradient of the expected return. Most policy gradient methods can be viewed as using a variant of the following gradient estimator
\begin{align}\label{eq:pg}
\E_\theta\Mp{\nabla_\theta\log\pi_\theta(a_t|o_t)\hat{A}_t(o_t,a_t)},
\end{align}
where $\hat{A}_t$ is the estimated advantage function at timestep $t$. The expectation is taken over the randomness of the environment and the policy. Here $\E_\theta$ means that $a_t$ is generated according to $\pi_\theta(\cdot|o_t)$.
Well-known practical algorithms such as PPO~\citep{schulman2017proximal} and GRPO~\citep{shao2024deepseekmath} can be understood as introducing modifications to this update in order to improve stability and sample efficiency.

Estimating~\Cref{eq:pg} becomes intractable when $\pi_\theta$ is parameterized by a diffusion or flow model, motivating a range of proxy objectives, including those based on variational bounds.

\paragraph{Variational Bounds.}
The evidence lower bound (ELBO) satisfies
\begin{align}\label{eq:elbo}
    \elbo_\theta(a_t|o_t)=\log\pi_\theta(a_t|o_t)-\gD_\theta^\text{L}(o_t,a_t)
\end{align}
where $\gD_\theta^\text{L}$ is a non-negative discrepancy term that measures the gap between the ELBO and true log-likelihood.
Furthermore, with appropriate and well-trained $\theta$, $\gD_\theta^\text{L}(o_t,a_t)=0$ for all $o_t$ and $a_t$.\footnote{Here we assume that $\theta$ has enough representation power to cover the zero discrepancy case.}
This means that ELBO is a tight lower bound of the true log-likelihood.
Outside the reinforcement learning context, generative models are usually trained by maximizing
\begin{align}\label{eq:objective}
    \E_{o_t,a_t\sim p_\text{data}}\Mp{\elbo_\theta(a_t|o_t)}
\end{align}
where $p_\text{data}$ is the distribution that the model tries to learn.
Such surrogate objectives are useful because they avoid direct likelihood computation. For diffusion models, under idealized assumptions, perfect training makes the ELBO tight, so $\elbo_\theta(a_t|o_t)=\log\pi_\theta(a_t|o_t)$. For flow models trained by conditional flow matching (CFM), the analogous idealized optimum has zero CFM loss, so the learned vector field matches the target conditional vector field.

Sometimes it is also useful to consider the corresponding upper bound on log likelihood. The evidence upper bound (EUBO) satisfies
\begin{align*}
    \eubo_\theta(a_t|o_t)=\log\pi_\theta(a_t|o_t)+\gD_\theta^\text{U}(o_t,a_t)
\end{align*}
where $\gD_\theta^\text{U}$ is a non-negative discrepancy term that can reach zero with appropriate $\theta$ as well. While $\gD_\theta^\text{L}$ can be interpreted as a KL discrepancy between posterior distributions, $\gD_\theta^\text{U}$ is characterized by a Renyi-type posterior discrepancy. Furthermore, if $\elbo_\theta(a_t|o_t)=\log\pi_\theta(a_t|o_t)$ for pair $(o_t,a_t)$, EUBO also satisfies $\eubo_\theta(a_t|o_t)=\log\pi_\theta(a_t|o_t)$. Hence, if a diffusion model is perfectly trained, we have $\elbo_\theta(a_t|o_t)=\eubo_\theta(a_t|o_t)=\log\pi_\theta(a_t|o_t)$ for all $o_t$ and $a_t$. As opposed to the ELBO, however, EUBO itself is not tractable for diffusion models and must be approximated in practice.

\paragraph{Adequate Gradient Estimators.}
Variational bounds provide attractive surrogates for the intractable log-likelihood because they can be tight. As we will see in the following sections, some variational-inference-based (VI-based) gradient estimators satisfy a useful property in this tight regime: once the surrogate coincides with the true log-likelihood, the estimator also coincides with the policy-gradient integrand in~\Cref{eq:pg}. We call such estimators adequate.
\begin{definition}[Adequate estimator]\label{def:adequate}
A gradient estimator $g_\theta(o_t,a_t)$ is called \emph{adequate} if, for any $(o_t,a_t)$, whenever the evidence bound is tight so that $\elbo_\theta(a_t|o_t)=\log\pi_\theta(a_t|o_t)$, it satisfies
\begin{align*}
    g_\theta(o_t,a_t)=\hat{A}_t(o_t,a_t)\nabla_\theta\log\pi_\theta(a_t|o_t).
\end{align*}
\end{definition}

\section{Method}
As discussed in the previous section, diffusion policies allow efficient sampling but make the exact log-likelihood $\log \pi_\theta(a_t| o_t)$ (and its gradient) intractable, forcing RL for diffusion models to rely on proxy objectives/estimators.
Our introduction diagnosed a \textbf{double drift} mechanism that is particularly acute for \emph{variational-inference (VI)} based diffusion RL:
(i) the \textcolor{B}{ELBO can drift from the true log-likelihood} as the variational discrepancy grows, and consequently
(ii) the \textcolor{A}{proxy policy gradient will drift apart from the true policy gradient of expected return}.
In this section, we first use this double-drift lens to clarify limitations of representative diffusion-RL algorithms (\Cref{sec:limitation}).
We then develop a principled framework that \emph{prevents drifting} by keeping the evidence bound tight on-policy, which in turn keeps policy-gradient updates aligned with the true policy gradient.
Finally, we derive a simple practical algorithm that implements this principle as a drop-in regularization term.

\subsection{The Double-Drift Phenomenon}\label{sec:limitation}
Many diffusion-RL methods differ in formulation, but the key question under our lens is:
\emph{does the proxy remain faithful to $\log \pi_\theta$, and does its induced gradient remain faithful to the true policy gradient?}
For VI-based approaches, the core issue is that an evidence bound update can change both the likelihood term and the variational gap, and once the gap grows, the proxy gradient inevitably deviates from $\nabla_\theta \log \pi_\theta$.

\paragraph{FPO~\citep{mcallister2025flow}.}
FPO replaces the intractable likelihood score in the policy-gradient update with an ELBO score. Formally, FPO uses the gradient estimator
\begin{align}\label{eq:fpo}
    g_\theta^\text{FPO}(o_t,a_t)=\nabla_\theta\elbo_\theta(a_t|o_t)\hat{A}_t(o_t,a_t),
\end{align}
up to PPO-style clipping for stability.
The intuition is that the ELBO acts as a tractable proxy for the log-likelihood score: in particular, increasing the ELBO on positive-advantage samples should both increase the likelihood of desirable actions and push the diffusion model closer to a tight-bound regime.
However, since ELBO decomposes as the true log-likelihood minus a nonnegative discrepancy term (cf.~\Cref{eq:elbo}), ELBO changes generally \emph{do not uniquely determine} how $\log\pi_\theta(a_t|o_t)$ changes.
This creates the first drift: \textcolor{B}{ELBO--likelihood inconsistency}.
Concretely:
\begin{itemize}
    \item The ELBO may increase even if $\log\pi_\theta(a_t|o_t)$ decreases, provided the discrepancy decreases more.
    This ambiguity is especially problematic for negative-advantage updates: the ELBO can be reduced by \emph{increasing} the discrepancy, i.e., ``cheating'' by destabilizing the diffusion model rather than reliably decreasing the true likelihood of undesirable actions.
    \item Once the discrepancy is non-negligible, the second drift follows automatically:
    the proxy gradient used in FPO is $\nabla_\theta \elbo_\theta(a_t|o_t)$, but $\nabla_\theta \elbo_\theta = \nabla_\theta \log\pi_\theta - \nabla_\theta \gD_\theta^\text{L}$.
    Thus, the update direction can become partially spent on changing the variational gap rather than following $\nabla_\theta \log\pi_\theta$, and \textcolor{A}{the induced policy-gradient step may not align with the true policy gradient in~\Cref{eq:pg}.}
\end{itemize}
These issues help explain why ELBO-based diffusion RL can be unstable in practice: once ELBO drifts from likelihood, the proxy gradient can drift from the true policy gradient as well.

\paragraph{SPG~\citep{wang2025spg}.}
SPG explicitly targets an objective that mirrors the policy-gradient structure by using ELBO for positive advantages and EUBO for negative advantages.
Formally, SPG uses the gradient estimator
\begin{align}\label{eq:spg}
    g_\theta^\text{SPG}(o_t,a_t)=\mathbbm{1}_{\hat{A}>0}\nabla_\theta\elbo_\theta(a_t|o_t)\hat{A}+\mathbbm{1}_{\hat{A}<0}\nabla_\theta\eubo_\theta(a_t|o_t)\hat{A},
\end{align}
where $\hat{A}$ abbreviates $\hat{A}_t(o_t,a_t)$.
The objective corresponding to the SPG update~(\Cref{eq:spg}) is a lower bound on the original objective:
\begin{align}\label{eq:spg_obj}
    \E_\theta\Mp{\mathbbm{1}_{\hat{A}>0}\elbo_\theta(a_t|o_t)\hat{A}+\mathbbm{1}_{\hat{A}<0}\eubo_\theta(a_t|o_t)\hat{A}}
\end{align}
because ELBO is a lower bound and EUBO is an upper bound on the true log-likelihood.
Moreover, SPG has a nice property: when the objective is maximized, \textcolor{B}{$\elbo$ and $\eubo$ equal the log-likelihood}, and the expected return is maximized as well.
However, under the double-drift lens it still has drawbacks:
\begin{itemize}
    \item Similar to the exact likelihood, EUBO is not tractable. The SPG implementation adopts an approximation to EUBO such that~\Cref{eq:spg_obj} remains a true lower bound. This lower-bound property often stabilizes training, but the approximation can break the exact link to the idealized objective-level argument and is currently specialized to certain settings.
    \item Most importantly for our purposes: an objective-level guarantee does \textbf{not} imply \textcolor{A}{gradient consistency}.
    Even if an objective becomes exact at convergence, there is no guarantee that intermediate gradient ascent steps computed from variational bounds align with the true policy gradient in~\Cref{eq:pg} when the bound is not tight (i.e., when the discrepancy is non-negligible).
\end{itemize}

\subsection{Preventing Double Drift through Self-Distillation}\label{sec:self_distillation}
We now make the core principle precise: to prevent the \textcolor{A}{second drift (proxy-gradient drift)}, we must prevent the \textcolor{B}{first drift (ELBO--likelihood drift)} from accumulating along the training trajectory.
Concretely, when the diffusion model is perfectly trained (or close to it), the evidence bound is tight and gradients of variational bounds coincide with gradients of the true log-likelihood.

Let us take ELBO as an example.
Recall from~\Cref{eq:elbo} that ELBO is a tight lower bound of the true likelihood and the discrepancy $\gD_\theta^\text{L}(o_t,a_t)$ is minimized at zero.
Assuming $\gD_\theta^\text{L}(o_t,a_t)$ is smooth in $\theta$, at a tight-bound point we have $\nabla_\theta\gD_\theta^\text{L}(o_t,a_t)=0$, and thus
\begin{align*}
    \elbo_\theta(a_t|o_t)=\log\pi_\theta(a_t|o_t),
    \Rightarrow\nabla_\theta\elbo_\theta(a_t|o_t)=\nabla_\theta\log\pi_\theta(a_t|o_t).
\end{align*}
Hence, $g^{\text{FPO}}_\theta$ is adequate, and if we start from a well-trained diffusion model (e.g., optimized under~\Cref{eq:objective}), the \emph{initial} ELBO-based update direction in FPO aligns with the true policy gradient in~\Cref{eq:pg}, and it remains a good approximation as long as training stays near the tight-bound regime.

A similar logic applies to SPG.
When the diffusion model is perfectly trained, $\gD_\theta^\text{L}(o_t,a_t)=\gD_\theta^\text{U}(o_t,a_t)=0$, and hence
\begin{align*}
    \nabla_\theta\elbo_\theta(a_t|o_t)=\nabla_\theta\eubo_\theta(a_t|o_t)=\nabla_\theta\log\pi_\theta(a_t|o_t),
\end{align*}
meaning that $g^{\text{SPG}}_\theta$ is adequate, and this justifies the strong empirical performance of VI-based RL methods for diffusion models when initialized from a pretrained model: their gradients are \emph{adequate} at initialization.
The core issue is that RL updates can move the model away from this tight-bound regime, allowing the discrepancy to grow; once that happens, ELBO drifts from likelihood and the proxy gradient drifts from the true policy gradient.

\paragraph{Key idea: repeatedly tighten the bound on-policy.}
We propose to \emph{interleave} policy-gradient updates (using an adequate estimator) with a policy-preserving self-distillation step that tightens ELBO under a \emph{reference rollout distribution}.
Intuitively, self-distillation directly counters the diffusion-side drift by reducing the discrepancy on-policy; this in turn restores the local tightness that makes the proxy gradient a faithful substitute for $\nabla_\theta \log \pi_\theta$, preventing the RL-side drift.

\begin{algorithm}[H]
\caption{DiPOD (Interleave)}
\label{alg:interleave}
\footnotesize
\begin{algorithmic}[1]
\REQUIRE An adequate gradient estimator $g$, a policy parameterized by diffusion model $\pi_\theta$, $n\in\sN^+$, learning rate $\eta\in\sR^+$
\ENSURE Policy $\pi_\theta$
\STATE Initialize policy $\pi_\theta$
\STATE Set $\pi_\text{ref}\gets\pi_\theta$
\WHILE{$\pi_\theta$ has not converged}
    \STATE Maximize $\E_\text{ref}\Mp{\elbo_\theta(a_t|o_t)}$.\hfill $\triangleright$ Self-distillation
    \FOR{$i=1,\dots,n$}
        \STATE $\theta\gets\theta+\eta\E_\theta\Mp{g_\theta(o_t,a_t)}$. \hfill $\triangleright$ Policy update
    \ENDFOR
    \STATE Set $\pi_\text{ref}\gets\pi_\theta$
\ENDWHILE
\STATE \textbf{return} $\pi_\theta$
\end{algorithmic}
\end{algorithm}

\Cref{alg:interleave} alternates between (i) tightening the evidence bound \emph{under rollouts from the latest reference policy} ($\E_\text{ref}$ means $a_t\sim\pi_\text{ref}(\cdot| o_t)$), and (ii) taking policy updates using an adequate estimator.
In the idealized limit where the self-distillation step is optimized to convergence, it produces a diffusion model with (approximately) zero discrepancy on the reference rollout distribution while preserving the reference policy’s output distribution; the subsequent policy update is then well-aligned with the true policy-gradient direction.
Crucially, we refresh $\pi_\text{ref}$ to be the \emph{most recent} policy, so the self-distillation step continually patches on-policy drift rather than distilling toward a stale reference.

In particular, under standard smoothness, realizability and optimization assumptions, we can theoretically prove the following theorem.
\begin{theorem}[Informal]\label{thm:informal_dipod}
Assume every self-distillation step returns a model that is within $\varepsilon$ of the optimal ELBO value on the current reference distribution. Choosing a sufficiently small learning rate $\eta$ ensures that \textcolor{A}{each DiPOD policy update improves the expected return}, unless the current policy gradient is already $O(\sqrt{\varepsilon})$.
If the procedure stops after an $\varepsilon$-optimal on-policy self-distillation step at $\theta_\star$, then both $\|\nabla_\theta\gJ(\theta_\star)\|=O(\sqrt{\varepsilon})$ and \textcolor{B}{the ELBO discrepancy under the final reference distribution is small:
$\E_{(o,a)\sim\rho_{\theta_\star}}[\gD_{\theta_\star}^{\mathrm{L}}(o,a)]\le\varepsilon$.}
\end{theorem}
We defer the formal statement and proof to the appendix.

\subsection{A Simple and Practical Implementation}
Although~\Cref{alg:interleave} directly addresses double drift, a literal implementation can be inefficient if self-distillation is run to convergence before every policy update.
We therefore introduce a simple approximation that is efficient, easy to implement, and effective in practice.

In a typical diffusion RL algorithm (such as FPO and SPG), each gradient step uses a batch of rollouts $\gB=\Bp{(o^i,a^i)}_{i=1,\dots,m}$ and updates parameters as
\begin{align*}
    \theta\gets\theta+\eta\cdot\frac1m\sum_{i=1}^mg_\theta(o^i,a^i).
\end{align*}
We absorb self-distillation into each update by augmenting the policy-update direction with an ELBO maximization term computed on the same rollout batch:
\begin{align}\label{eq:implement}
    \theta\gets\theta+\eta\cdot\frac1m\sum_{i=1}^m\Mp{g_\theta(o^i,a^i)+\beta\nabla_\theta\elbo_\theta(a^i|o^i)},
\end{align}
where $\beta\in\sR^+$ controls the strength of the on-policy ELBO tightening.

\Cref{eq:implement} can be viewed as a first-order approximation to performing (i) a small amount of self-distillation and (ii) a policy-gradient update per rollout batch, sharing the same samples for efficiency.
Equivalently, the added ELBO term acts as a regularizer that reduces the variational discrepancy $\gD_\theta^\text{L}$ over on-policy rollouts, tightening the likelihood approximation and thereby mitigating the \textcolor{B}{\emph{first drift} (ELBO--likelihood drift)}.
As a consequence, the proxy gradient used in $g_\theta$ remains better aligned with $\nabla_\theta\log\pi_\theta$, mitigating the \textcolor{A}{\emph{second drift} (proxy-gradient drift)}.

We summarize the resulting procedure in \Cref{alg:implementation}.
Relative to a standard VI-based diffusion RL pipeline, the only modification is the extra ELBO term in the $\theta$ update, making the method a simple drop-in enhancement for a broad class of variational-bound-based diffusion RL algorithms.

\begin{algorithm}[H]
\caption{DiPOD (Practical Implementation)}
\label{alg:implementation}
\footnotesize
\begin{algorithmic}[1]
\REQUIRE An adequate gradient estimator $g$, a policy parameterized by diffusion model $\pi_\theta$, $\beta\in\sR^+$, learning rate $\eta\in\sR^+$, rollout size $m\in\sN^+$
\ENSURE Policy $\pi_\theta$
\STATE Initialize policy $\pi_\theta$
\WHILE{$\pi_\theta$ has not converged}
    \STATE Sample rollouts $\gB=\Bp{(o^i,a^i)}_{i=1,\dots,m}$\\
    $\triangleright$ Rollouts may be sampled from the current policy.
    \STATE Update $\theta$ using~\Cref{eq:implement} with $\gB$.
\ENDWHILE
\STATE \textbf{return} $\pi_\theta$
\end{algorithmic}
\end{algorithm}

Although we present our method in the context of diffusion policies, the proposed principle is not specific to diffusion models.
At a high level, our algorithm only requires (i) a tractable variational lower bound of the form $\mathrm{ELBO}_\theta(x|c)$ and (ii) an RL objective that depends on an intractable likelihood only through policy-gradient-style updates.
Consequently, the same ``tighten-the-bound to prevent gradient drift'' strategy applies to a broad class of generative policies trained with ELBO objectives, including latent-variable models such as VAEs and other variational generative models.

\section{Experiment}\label{sec:exp}
Code is released at \href{https://github.com/Astro-Eric/DiPOD-release}{Astro-Eric/DiPOD-release}.
In this section, we illustrate the effectiveness and versatility of~\Cref{alg:implementation} across both discrete and continuous diffusion-policy domains.
We first use a Two-Token post-training experiment as a controlled diagnostic that makes ELBO drift directly visible.
We then evaluate DiPOD on two substantially different practical settings: RL post-training of diffusion language models on challenging reasoning benchmarks, and high-dimensional continuous-control motion tracking for a humanoid robot.
\subsection{Two-Token Post-training}

\noindent
\begin{minipage}[t]{0.57\textwidth}
\paragraph{Setup.} We use a two-token discrete diffusion model following the toy post-training setting of SPG~\citep{wang2025spg}. The model generates $x=(x_1,x_2)$ with each token in $\{\mathrm{A},\mathrm{B}\}$, starting from the masked state $\mathrm{MM}$ and decoding in a uniformly random order. Because the state space is fully enumerable, we can analytically compute $\log\pi$, ELBO, and the variational gap $\gD_\theta^\text{L}$ throughout training, which makes the drift directly visible. We defer the full parameterization, rewards, and implementation details to~\Cref{app:toy_details}.

\paragraph{Results.} We summarize the results in~\Cref{fig:toy}. For FPO, the variational gap grows substantially as training proceeds. For SPG, while staying more controlled, the gap is still substantial. Applying DiPOD to both algorithms effectively controls the gap.
\end{minipage}\hfill
\begin{minipage}[t]{0.38\textwidth}
    \vspace{-1.0em}
    \centering
    \includegraphics[width=\linewidth]{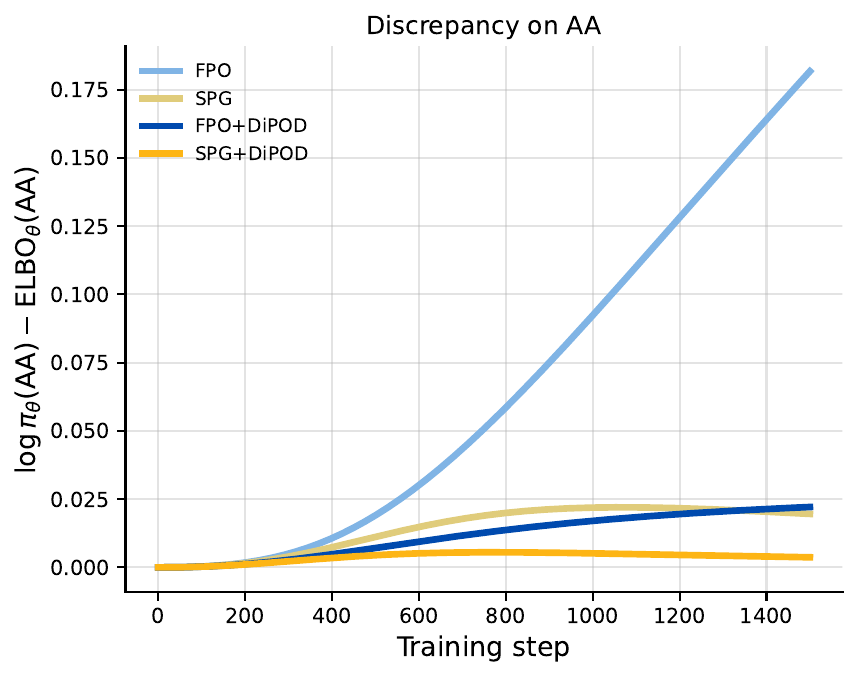}
    \captionsetup{type=figure,hypcap=false}
    \caption{Variational gap $\gD_\theta^\text{L}$ during FPO, SPG, and DiPOD training.}
    \label{fig:toy}
\end{minipage}

\FloatBarrier
\subsection{Reasoning Tasks with Diffusion Language Models}\label{sec:exp_language}
\paragraph{Setup.} We follow the basic experimental setups in d1~\citep{zhao2025d1} and SPG~\citep{wang2025spg}. We start from LLaDA-8B-Instruct~\citep{nie2025large}, an open-source pretrained diffusion large language model, and conduct RL experiments on it. We include four tasks: GSM8K~\citep{cobbe2021training}, MATH500~\citep{lightman2023let}, Countdown~\citep{pan2025tinyzero}, and Sudoku~\citep{arelSudokuGenerator}.

\paragraph{Baselines and Hyperparameters.} We evaluate the performance of~\Cref{alg:implementation} with the FPO gradient estimator and with the SPG gradient estimator separately. For FPO experiments, we compare d1, FPO, and DiPOD (\Cref{alg:implementation}) with the FPO gradient estimator. We implement FPO by replacing the log likelihoods in d1 with ELBOs, resulting in a GRPO version of FPO, while keeping the hyperparameters unchanged. We also keep the hyperparameters in DiPOD the same as d1. Similarly, for SPG experiments, we compare d1, SPG, and~\Cref{alg:implementation} with the SPG gradient estimator. We use SPG with Mixture, which performs the best among all variants in the SPG paper. For DiPOD, we keep the hyperparameters the same as SPG. \textbf{For all language experiments, we fix $\beta=0.05$ in~\Cref{eq:implement}, highlighting the consistent improvements of DiPOD over baseline algorithms.} We include an ablation study on $\beta$ and additional results with more sequence lengths in the appendix. We fix the sequence length of the diffusion language model to be 256, and the number of decoding steps to be 128. For all benchmarks, we evaluate the performance in the zero-shot setting.

\paragraph{Results.} We summarize results in~\Cref{tab:reasoning_aggregate}, and show reward dynamics in~\Cref{fig:reward_dynamics}. The table and figures show that FPO+DiPOD consistently improves the performance over FPO. Compared to SPG, the state-of-the-art algorithm in diffusion language model RL, SPG+DiPOD, shows competitive performance in math reasoning tasks, including GSM8K and MATH500, and achieves a significant leap in logical reasoning tasks, including Countdown and Sudoku. For mathematical reasoning, we conjecture that the fixed context window of diffusion language models is the primary bottleneck for solving harder problems, which fundamentally require longer chains of thought. Consequently, we observe only marginal improvements on GSM8K and MATH500. We notice that the performance of the algorithms can differ substantially for different random seeds. To ensure the fairness of the comparison, we keep the random seeds the same as in the SPG codebase.

\begin{table*}[t]
\centering
\caption{Blue numbers denote improvements over the corresponding baseline (e.g., FPO+DiPOD vs. FPO, SPG+DiPOD vs. SPG). Evaluation follows the protocol of SPG~\citep{wang2025spg}. We adopt the reported results for d1 from the SPG paper, except for Sudoku, where we report results from the d1 paper~\cite{zhao2025d1}.}
\label{tab:reasoning_aggregate}
\begin{tabular}{lcccc}
\toprule
Method & GSM8K & MATH500 & Countdown & Sudoku \\
\midrule
d1          & 80.60 & 36.00 & 30.90 & 22.10 \\
FPO     & 81.50 & 36.60 & 12.50 & 6.88 \\
FPO+DiPOD
            & 83.24 (\textcolor{gain}{+1.74})
            & 37.00 (\textcolor{gain}{+0.40})
            & 75.39 (\textcolor{gain}{+62.89})
            & 25.78 (\textcolor{gain}{+18.90}) \\
SPG         & 84.23 & 37.80 & 51.95 & 25.12 \\
SPG+DiPOD   & \textbf{84.91} (\textcolor{gain}{+0.68}) & \textbf{40.00} (\textcolor{gain}{+2.20}) & \textbf{80.08} (\textcolor{gain}{+28.13}) & \textbf{97.56} (\textcolor{gain}{+72.44}) \\
\bottomrule
\end{tabular}
\end{table*}

\begin{figure}[!htbp]
    \centering
    \begin{subfigure}{0.24\textwidth}
        \centering
        \includegraphics[width=\linewidth]{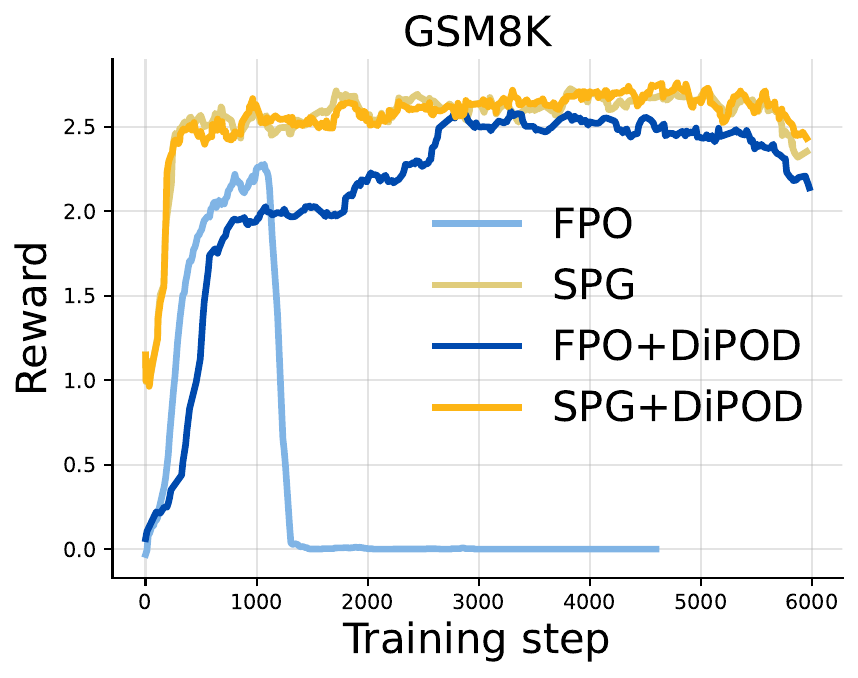}
    \end{subfigure}
    \hfill
    \begin{subfigure}{0.24\textwidth}
        \centering
        \includegraphics[width=\linewidth]{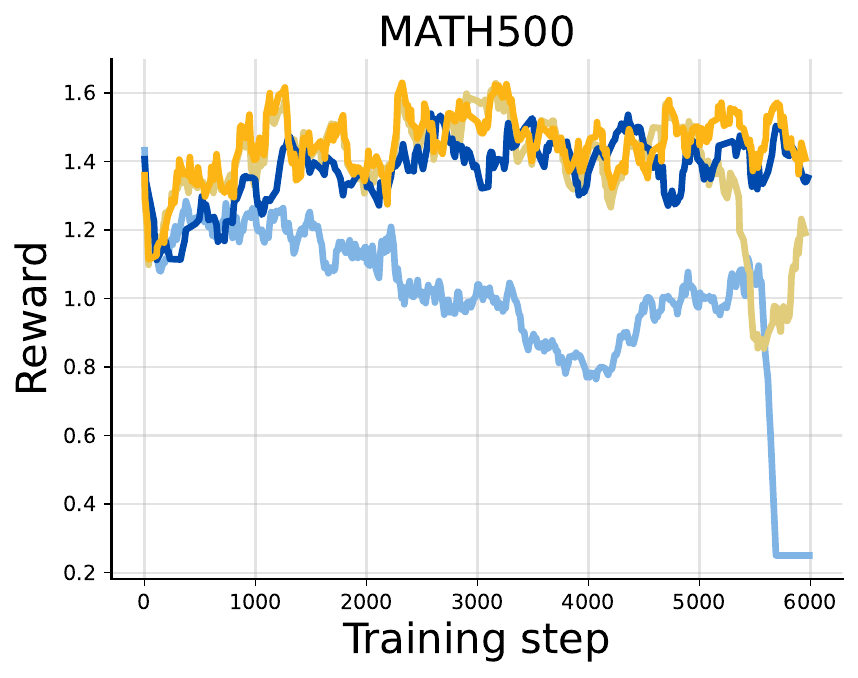}
    \end{subfigure}
    \hfill
    \begin{subfigure}{0.24\textwidth}
        \centering
        \includegraphics[width=\linewidth]{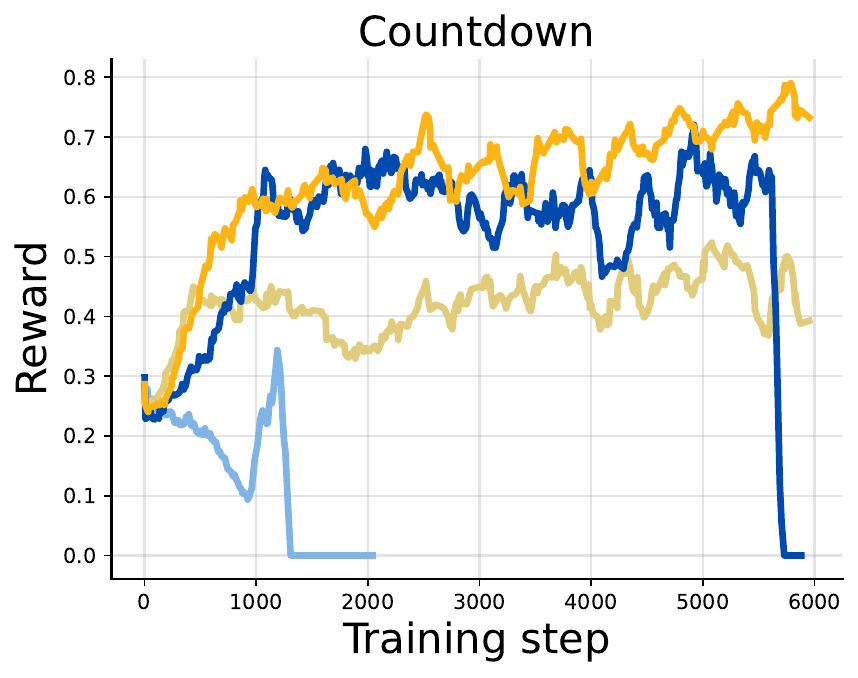}
    \end{subfigure}
    \hfill
    \begin{subfigure}{0.24\textwidth}
        \centering
        \includegraphics[width=\linewidth]{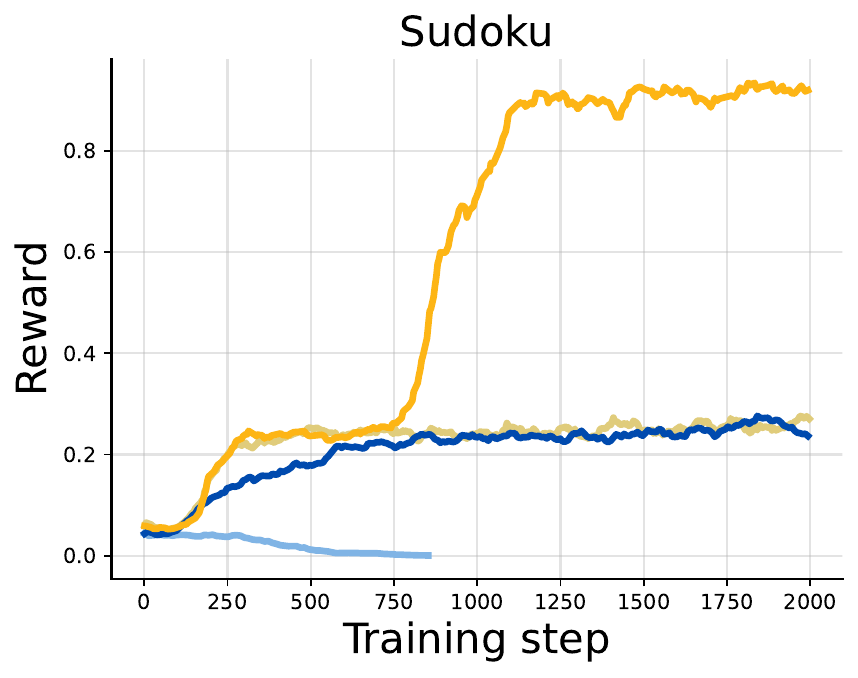}
    \end{subfigure}
    \vspace{-0.6em}
    \caption{Reward dynamics on reasoning tasks. DiPOD stabilizes FPO and SPG training and achieves competitive or better final performance.}
    \label{fig:reward_dynamics}
\end{figure}

\FloatBarrier

\subsection{Motion Tracking with Diffusion Policies}
\paragraph{Setup.} We further evaluate DiPOD on continuous-control diffusion policies through the motion-tracking task from FPO++~\citep{yi2026flow}, a variant of FPO that gives an adequate gradient estimator. This is a high-dimensional robot-control problem: the policy controls a Unitree G1 humanoid to track reference motions from the LAFAN dataset~\citep{harvey2020robust}. Unlike the language experiments, this setting is not a post-training benchmark: the flow policy is trained for motion tracking from the beginning.

\paragraph{Baselines and Hyperparameters.} We reuse the FPO++ motion-tracking hyperparameters. For DiPOD, we use a single initial self-distillation stage before the standard FPO++ policy-gradient training. This tightens the ELBO--likelihood gap while preserving the initial policy distribution, placing subsequent updates closer to the tight-bound regime predicted by our theory. We choose this minimal instantiation because policy-preserving self-distillation during high-dimensional motion-control training is itself a nontrivial algorithmic component; developing fully interleaved schedules for motion tracking is an interesting direction for future work.

\paragraph{Results.} \Cref{fig:motion_tracking} shows mean reward and episode length for \emph{dance\_1\_subject\_2} and \emph{run\_1\_subject\_2}. DiPOD improves reward dynamics and tracking duration over reproduced FPO++, supporting the same principle beyond diffusion language models. Additional details are in the appendix. 

\begin{figure}[!htbp]
    \centering
    \begin{subfigure}[t]{0.24\textwidth}
        \centering
        \includegraphics[height=0.75\linewidth,keepaspectratio]{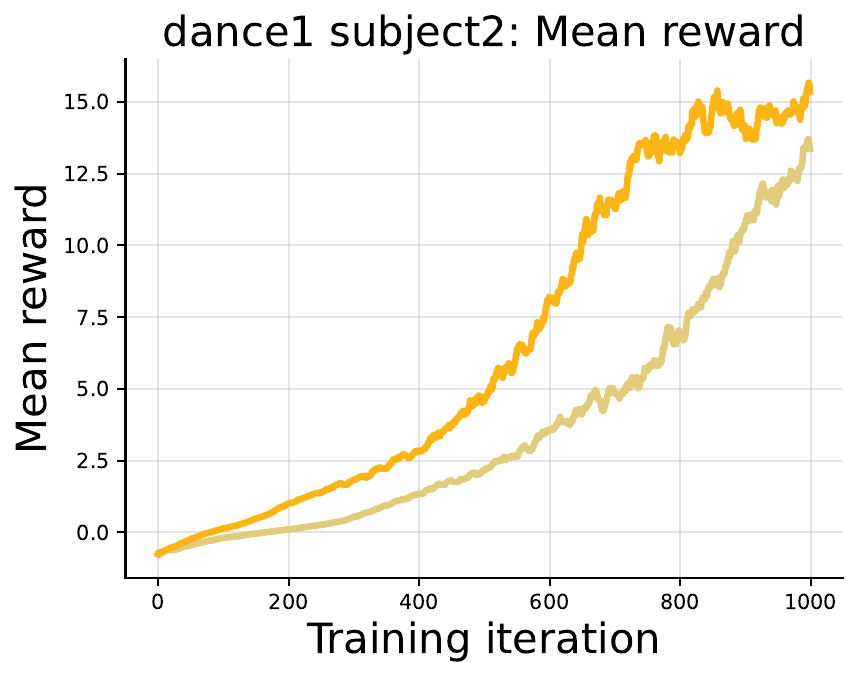}
    \end{subfigure}
    \hfill
    \begin{subfigure}[t]{0.24\textwidth}
        \centering
        \includegraphics[height=0.75\linewidth,keepaspectratio]{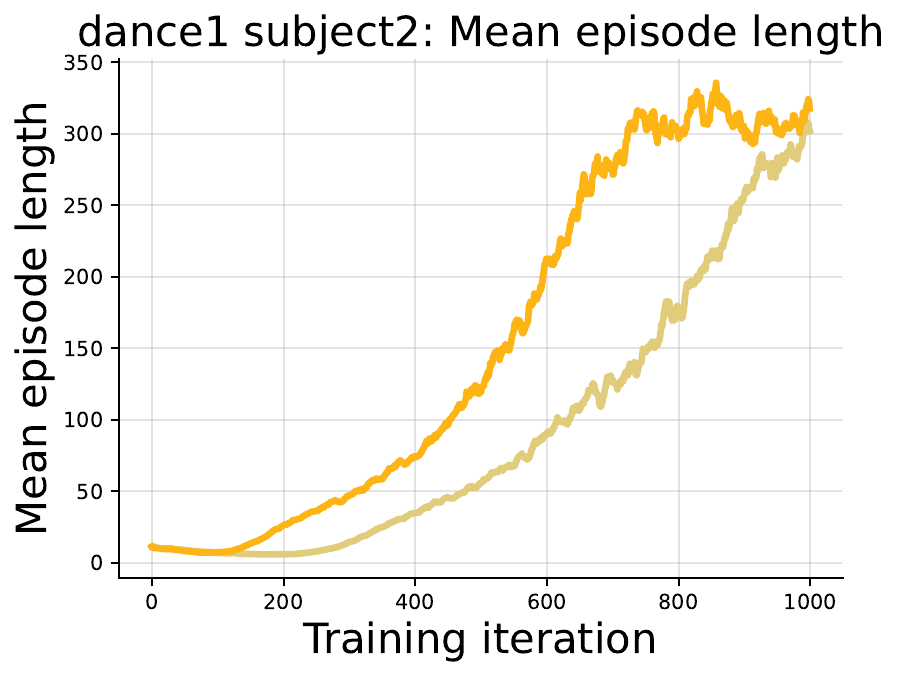}
    \end{subfigure}
    \hfill
    \begin{subfigure}[t]{0.24\textwidth}
        \centering
        \includegraphics[height=0.75\linewidth,keepaspectratio]{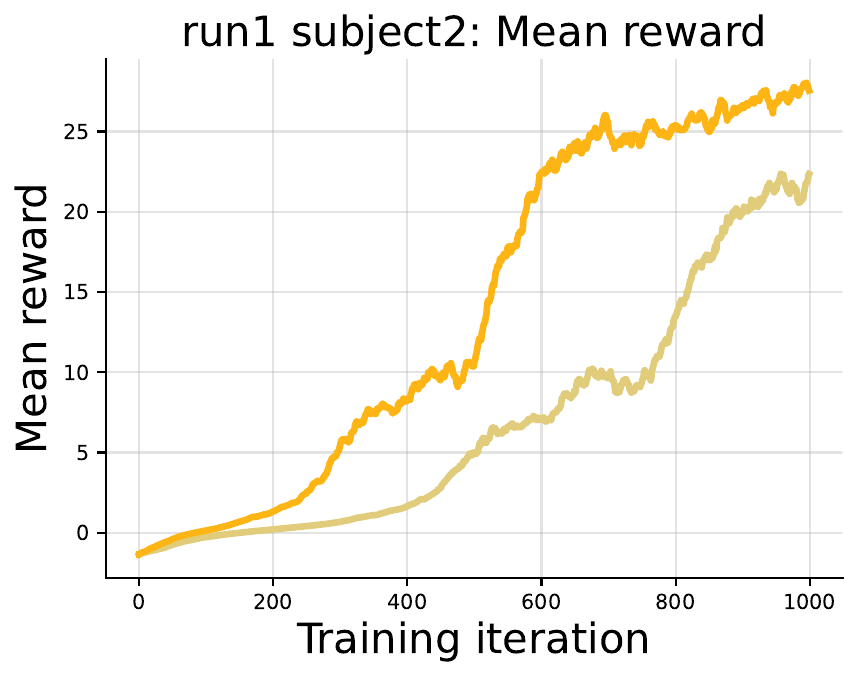}
    \end{subfigure}
    \hfill
    \begin{subfigure}[t]{0.24\textwidth}
        \centering
        \includegraphics[height=0.75\linewidth,keepaspectratio]{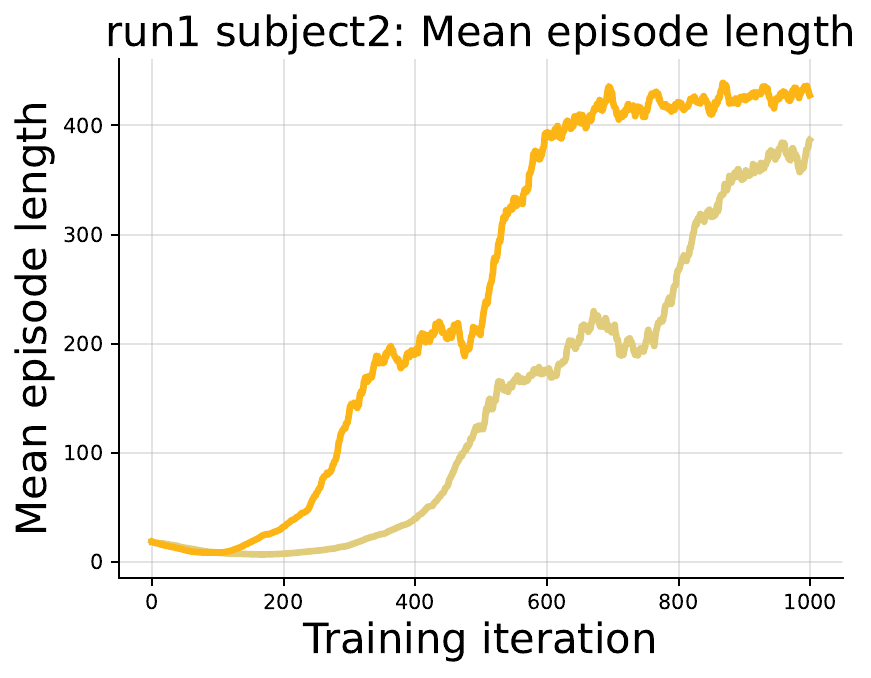}
    \end{subfigure}
    \vspace{0.25em}
    \begingroup
    \definecolor{motionbaseline}{RGB}{226,202,110}
    \definecolor{motiondipod}{RGB}{255,170,11}
    \small
    \textcolor{motionbaseline}{\rule[0.45ex]{1.8em}{0.5ex}}\;FPO++
    \qquad
    \textcolor{motiondipod}{\rule[0.45ex]{1.8em}{0.5ex}}\;DiPOD
    \par
    \endgroup
    \vspace{0.2em}
    \caption{Motion-tracking reward and episode-length curves on \emph{dance\_1\_subject\_2} and \emph{run\_1\_subject\_2}.}
    \label{fig:motion_tracking}
\end{figure}

\section{Conclusion}
We identify a double drift issue in diffusion reinforcement learning: loose variational bounds induce proxy gradients that drift from the true policy gradient, undermining policy improvement. DiPOD preserves gradient consistency by maintaining on-policy bound tightness. Its most notable empirical outcome is substantially improved reasoning capability and training stability for diffusion language models, with additional gains in continuous control. Future work should optimize the interleaving schedule and study alternatives to ELBO regularization.

\section*{Acknowledgment}
We thank Banghua Zhu, Chenyu Wang, David McAllister, Hongsuk Choi, Brent Yi, and Yen-Jen Wang for helpful conversations.
This work was supported in part by the National Science Foundation under grants CCF-2145898 and CCF-2505865, by the Office of Naval Research under grant N00014-24-1-2159, a Google Research Scholar Award, an Alfred P. Sloan fellowship, and a Schmidt Science AI2050 fellowship.

\bibliography{reference}
\bibliographystyle{plainnat}

\newpage
\appendix
\section{Additional Related Work}\label{app:related_work}
\paragraph{Diffusion models.} Diffusion models are a powerful class of generative models, with wide applications including image generation~\citep{ho2020denoising,song2020denoising,song2019generative}, video generation~\citep{ho2022video,ho2022imagen}, and robotics~\citep{black2410pi0,bjorck2025gr00t}. More recently, diffusion language models have emerged as a strong alternative to autoregressive models for fast inference and flexible decoding~\citep{nie2025large,xie2025dream,song2025seed}. Diffusion models can be derived from several viewpoints; in this paper, we use the variational-inference perspective~\citep{kingma2023understanding,sohl2015deep}, which connects the intractable likelihood to tractable evidence bounds.

\paragraph{Policy gradients.} In reinforcement learning, policy-gradient methods optimize a parameterized policy by directly differentiating its expected cumulative return~\citep{williams1992simple,schulman2017proximal}. They have led to strong results in high-dimensional control~\citep{schulman2016high}, locomotion~\citep{rudin2022learning}, manipulation~\citep{schwarke2023curiosity}, and related sequential decision-making problems. Policy-gradient-style methods also play a central role in the post-training of large language models~\citep{shao2024deepseekmath,rafailov2023direct}, making their extension to diffusion language models an important algorithmic question.

\paragraph{RL with diffusion policies.} Classical policy-gradient methods cannot be directly applied to diffusion models because the exact likelihood is intractable. One line of work tackles this problem by treating the denoising process as a Markov Decision Process~\citep{black2023training,ren2024diffusion}. This makes training tractable, but ties the optimization procedure to a particular denoising sampler. Another line of work uses variational objectives, such as diffusion or flow-matching evidence bounds, to bypass exact likelihood computation while preserving the native diffusion sampling process~\citep{mcallister2025flow,wang2025spg}. In language domains, several methods instead approximate likelihoods by simplifying or partially restoring inter-token dependencies~\citep{zhao2025d1,yang2025mmada,xie2025dream,tang2025wd1}. DiPOD is closest to the variational line: it identifies how evidence-bound looseness can corrupt policy-gradient updates, and introduces self-distillation as a mechanism for keeping the proxy aligned with the true policy.

\paragraph{Other diffusion-policy RL and control methods.} Beyond online policy-gradient methods, there is a rich literature on critic-based or residual-control methods for diffusion policies. Some works leave a pretrained diffusion policy unchanged and post-process generated actions using a learned critic~\citep{hansen2023idql,mark2024policy,li2025reinforcement,dong2025expo}. Others learn a residual policy that modifies the generated action in noise space~\citep{wagenmaker2025steering,singh2020parrot} or action space~\citep{yuanpolicy,ankile2025imitation}. A related set of off-policy methods uses a critic to supervise the score or denoising function directly~\citep{fangdiffusion,ding2024diffusion,li2026q}. These methods address complementary settings, while DiPOD focuses on stabilizing online policy-gradient updates when the diffusion-policy likelihood must be replaced by a variational proxy.

\paragraph{Guidance and stochastic-control perspectives.} Guidance is also related to RL with diffusion policies, since the guided sampler can be interpreted as tilting a base distribution according to reward or energy. Early guidance approaches can suffer from bias in the resulting sampling distribution~\citep{du2023reduce,karras2024guiding}. Adjoint Matching~\citep{domingoadjoint} takes a stochastic optimal-control perspective and derives an efficient fine-tuning objective that converges to the desired tilted distribution under its assumptions. DiPOD instead studies the policy-gradient setting where rewards may be non-differentiable and where the central challenge is maintaining alignment between the evidence-bound proxy and the true diffusion-policy likelihood.

\section{Background on Reinforcement Learning and Diffusion Models}\label{sec:rl_diffusion}
\subsection{Reinforcement Learning}
This section provides a concise introduction to the RL algorithms relevant to this paper.
\paragraph{Markov Decision Process.} A Markov Decision Process (MDP) consists of a state space $\gS$, an action space $\gA$, a state transition kernel $P$, and a reward function $r$. A policy is defined as a function $\pi(\cdot|s)$ that outputs a distribution on $\gA$ given a state $s$ from $\gS$. In the MDP, an agent with policy $\pi$ starts from a designated state $s_0\in\gS$ at timestep $t=0$. At each timestep $t$, the policy receives the current state $s_t$ and takes an action $a_t\sim\pi(\cdot|s_t)$. The environment transitions to state $s_{t+1}\sim P(\cdot|s_t,a_t)$, and gives the agent a reward $r_t=r(s_t,a_t)$. The episode terminates when the agent reaches a terminal state $s^*\in\gS$. An RL algorithm aims to optimize the cumulative reward, defined as
\begin{align*}
    \gJ(\theta)=\E_{a_{t}\sim\pi(\cdot|s_t),s_{t+1}\sim P(\cdot|s_t,a_t)}\Mp{\sum_{t\geq0}\gamma^tr_t}
\end{align*}
where $\gamma\in(0,1)$ is a discount factor. The value function is defined as the expected cumulative reward starting from a state $s$. Formally:
\begin{align*}
    V_t^\pi(s)=\E_{\pi}\Mp{\sum_{\tau\geq t}\gamma^{\tau-t}r_\tau|s_t=s}
\end{align*}
where $\E_{\pi}$ means that the actions are sampled from policy $\pi$. The $Q$-function is defined as the expected cumulative reward starting from state $s$ and action $a$. Formally:
\begin{align*}
    Q_t^\pi(s,a)=\E_{\pi}\Mp{\sum_{\tau\geq t}\gamma^{\tau-t}r_\tau|s_t=s,a_t=a}
\end{align*}
\paragraph{Policy Gradient Algorithms.} The policy gradient theorem states that
\begin{align*}
    \nabla_\theta\gJ(\theta)=\E_{a_t\sim\pi_\theta(\cdot|s_t),s_t\sim d^{\pi_\theta}}\Mp{\nabla_\theta\log\pi_\theta(a_t|s_t)Q_t^{\pi_\theta}(s_t,a_t)}
\end{align*}
where $d^\pi$ is the state distribution under policy $\pi$. GAE~\citep{schulman2016high,greensmith2004variance} introduces an alternative form of the policy gradient that admits lower estimation variance:
\begin{align*}
    \nabla_\theta\gJ(\theta)=\E_{\theta}\Mp{\nabla_\theta\log\pi_\theta(a_t|s_t)A_t^{\pi_\theta}(s_t,a_t)}
\end{align*}
where the advantage function is defined as $A_t^\pi(s,a)=Q_t^\pi(s,a)-V^\pi_t(s)$ and we abbreviate the expectation subscript for simplicity.

In the off-policy setting, where a behavior policy $\pi_\text{ref}$ is used to collect data, the algorithm aims to optimize a slightly different objective defined as
\begin{align*}
    \gJ'(\theta)=\E_{s_t\sim d^{\pi_\text{ref}}}\Mp{V^{\pi_\theta}_t(s_t)}.
\end{align*}
In this setting, the following policy gradient update:
\begin{align*}
    \E_\text{ref}\Mp{\frac{\nabla_\theta\pi_\theta(a_t|s_t)}{\pi_\text{ref}(a_t|s_t)}A_t^{\pi_\theta}(s_t,a_t)}
\end{align*}
guarantees policy improvement on $\gJ'(\theta)$. In practice, $A_t^{\pi_\theta}(s_t,a_t)$ is often estimated through a combination of rewards from the samples and a learned value function.

\paragraph{Proximal Policy Optimization (PPO).} PPO is an instantiation of the off-policy gradient update introduced above. It performs gradient updates through
\begin{align*}
    \nabla_\theta\E_\text{ref}\Mp{\min\Bp{r(\theta)\hat{A}_t,\text{clip}(r(\theta),1\pm\varepsilon)\hat{A}_t}}
\end{align*}
where we abbreviate
\begin{align*}
    r(\theta)=\frac{\pi_\theta(a_t|s_t)}{\pi_\text{ref}(a_t|s_t)}
\end{align*}
and $\hat{A}_t$ is an estimation of $A_t^{\pi_\theta}(s_t,a_t)$. $\varepsilon$ is a hyperparameter between 0 and 1 that controls the strength of regularization to the policy update. The training data is sampled from $\pi_\text{ref}$, and $\pi_\text{ref}$ is synced with $\pi_\theta$ every few gradient steps. FPO~\citep{mcallister2025flow} builds on the PPO framework by substituting the log likelihoods of $\pi_\theta$ and $\pi_\text{ref}$ as their corresponding ELBOs.

\paragraph{Group Relative Policy Optimization.} GRPO is a variant of PPO for autoregressive language models. In language models, a prompt $q$, consisting of a sequence of tokens, serves as the input. The language model then generates another sequence of tokens $o$ as the output. This process can be framed as an MDP. At each timestep $t$, the state is the token sequence $q|o_{\leq t}$ that has been generated so far. Here $o_{\leq t}$ means the first $t$ tokens of $o$. The action $o_{t+1}$ is the next token to be output. In many scenarios the reward is given only when $o$ is fully generated, and the reward is often based only on the last few tokens, while the intermediate tokens can be thought of as reasoning traces. In GRPO, for each prompt $q$, we sample $g$ outputs $o^1,\dots,o_g$ using $\pi_\text{ref}$, and receive rewards $r^1,\dots,r^g$. We estimate the advantage as
\begin{align*}
    \hat{A}_{t}^i=\frac{r^i-\text{mean}(r^{1:g})}{\text{std}(r^{1:g})}.
\end{align*}
and estimate the policy-gradient objective for $q$ as
\begin{align*}
    \frac1g\sum_{i=1}^G\frac{1}{|o_i|}\sum_{t=1}^{|o_i|}\min\Bp{r(\theta)^i_t\hat{A}^i_t,\text{clip}(r(\theta)^i_t,1\pm\varepsilon)\hat{A}^i_t}
\end{align*}
where
\begin{align*}
    r(\theta)^i_t=\frac{\pi_\theta(o_t^i|q,o^i_{<t})}{\pi_\text{ref}(o_t^i|q,o^i_{<t})}.
\end{align*}
Compared to PPO, GRPO uses relative advantages inside a group of rollouts instead of the true advantage, eliminating the need to train a value function or calculate different value functions for different timesteps. Dream-Coder~\citep{xie2025dream} uses exactly the same algorithm for diffusion language models, ignoring that likelihoods cannot be decomposed in the same way as autoregressive models. \emph{Diffu}-GRPO~\citep{zhao2025d1}, wd1~\citep{tang2025wd1}, and UniGRPO~\citep{yang2025mmada} account for this factor, but still partially ignore inter-token dependencies for tractability.

\paragraph{FPO/FPO++.} FPO++~\citep{yi2026flow} is an improved flow-policy-gradient algorithm that replaces the single action-level FPO ratio with per-sample ratios and uses an asymmetric trust region. In our notation, if $\widehat{\elbo}_{\theta}^{(i)}(a_t|o_t)$ denotes the negative conditional-flow-matching loss for the $i$-th Monte Carlo sample, its per-sample ratio is
\begin{align*}
    \hat{\rho}^{(i)}_{\mathrm{FPO++}}(\theta)
    =
    \exp\!\left(\widehat{\elbo}_{\theta}^{(i)}(a_t|o_t)-\widehat{\elbo}_{\theta_{\mathrm{old}}}^{(i)}(a_t|o_t)\right).
\end{align*}
FPO++ then applies PPO clipping for $\hat{A}_t\ge0$ and an SPO-style quadratic trust region for $\hat{A}_t<0$. This gives an asymmetric interpretation. For positive-advantage samples, FPO++ keeps the original FPO intuition: it treats an ELBO lift as if it were a log-likelihood lift, so decreasing the conditional-flow-matching loss is taken to increase the likelihood of desirable actions. This is exact in the tight-bound regime, and otherwise is the same local approximation underlying FPO.

For negative-advantage samples, the SPO branch has a more explicit adaptive-regularization form. Recall that
\begin{align*}
    \psi_{\mathrm{SPO}}(\rho,\hat{A}_t)
    =
    \rho\hat{A}_t
    -
    \frac{|\hat{A}_t|}{2\varepsilon^{\mathrm{clip}}}(\rho-1)^2,
    \qquad
    \nabla_\theta \hat{\rho}^{(i)}_{\mathrm{FPO++}}
    =
    \hat{\rho}^{(i)}_{\mathrm{FPO++}}
    \nabla_\theta \widehat{\elbo}_{\theta}^{(i)}(a_t|o_t).
\end{align*}
Therefore, for $\hat{A}_t<0$,
\begin{align*}
    \nabla_\theta
    \psi_{\mathrm{SPO}}\!\left(\hat{\rho}^{(i)}_{\mathrm{FPO++}},\hat{A}_t\right)
    =
    \left[
    \hat{\rho}^{(i)}_{\mathrm{FPO++}}\hat{A}_t
    -
    \frac{|\hat{A}_t|}{\varepsilon^{\mathrm{clip}}}
    \hat{\rho}^{(i)}_{\mathrm{FPO++}}\!\left(\hat{\rho}^{(i)}_{\mathrm{FPO++}}-1\right)
    \right]
    \nabla_\theta \widehat{\elbo}_{\theta}^{(i)}(a_t|o_t),
\end{align*}
which is the usual ratio-weighted policy-gradient term plus an adaptive DiPOD-like ELBO term $\beta^{(i)}_-\nabla_\theta\widehat{\elbo}_{\theta}^{(i)}(a_t|o_t)$ with signed coefficient
\begin{align*}
    \beta^{(i)}_{-}(\theta,\hat{A}_t)
    =
    -
    \frac{|\hat{A}_t|}{\varepsilon^{\mathrm{clip}}}
    \hat{\rho}^{(i)}_{\mathrm{FPO++}}\!\left(\hat{\rho}^{(i)}_{\mathrm{FPO++}}-1\right),
\end{align*}
where the sign is for gradient ascent on $\psi_{\mathrm{SPO}}$; equivalently, the sign flips if one writes the implementation as minimizing $-\psi_{\mathrm{SPO}}$. Up to this sign convention and the multiplicative ratio factor, this is exactly a coefficient proportional to $|\hat{A}_t|(\hat{\rho}^{(i)}_{\mathrm{FPO++}}-1)/\varepsilon^{\mathrm{clip}}$. The term vanishes on-policy and grows with the ratio deviation, pulling negative-advantage ratios back toward $1$. This estimator is adequate in the sense of~\Cref{def:adequate}: when the bound is tight and the reference policy is synchronized with the current policy, $\hat{\rho}^{(i)}_{\mathrm{FPO++}}=1$, the adaptive term vanishes, and $\mathbb{E}_i[\nabla_\theta\widehat{\elbo}_{\theta}^{(i)}(a_t|o_t)]=\nabla_\theta\elbo_\theta(a_t|o_t)=\nabla_\theta\log\pi_\theta(a_t|o_t)$. Averaging over Monte Carlo samples therefore recovers the policy-gradient integrand $\hat{A}_t\nabla_\theta\log\pi_\theta(a_t|o_t)$.

\subsection{Diffusion Models}
This section provides a brief introduction to diffusion models from a variational inference perspective. We refer interested readers to~\cite{wang2025spg} for more details regarding diffusion language models.

\paragraph{Variational Inference.}
A central goal in generative modeling is to learn a parameterized distribution $p_\theta$ that assigns high probability to observed data. Given a dataset drawn from an unknown data distribution $p_\text{data}$, the standard training objective is maximum likelihood estimation (MLE),
\begin{align*}
\max_\theta \mathbb{E}_{x_0\sim p_\text{data}}[\log p_\theta(x_0)].
\end{align*}
This objective applies broadly, independent of the specific model family. In many modern generative models, including diffusion models, it is natural to introduce latent variables to describe a multi-step sampling procedure: one may first sample a latent variable $z$, then sample the data $x_0$ conditional on $z$. This leads to a latent-variable model of the form
\begin{align*}
p_\theta(x_0)=\int p_\theta(x_0,z)\mathrm{d} z,
\end{align*}
where $z$ can be high-dimensional (and in diffusion, typically corresponds to an entire trajectory of intermediate variables). While MLE still aims to maximize $\log p_\theta(x_0)$, the marginalization over $z$ often makes $\log p_\theta(x_0)$ intractable to evaluate and differentiate.

Variational inference addresses this intractability by introducing an auxiliary distribution $q(z|x_0)$, often called a variational posterior, that approximates the true posterior $p_\theta(z|x_0)$. For any choice of $q(z|x_0)$, Jensen's inequality gives
\begin{align}
&\log p_\theta(x_0) \nonumber\\
=&\log \int q(z|x_0)
\frac{p_\theta(x_0,z)}{q(z|x_0)}\mathrm{d} z\\
\ge&
\mathbb{E}_{q(z|x_0)}\left[
\log p_\theta(x_0,z)-\log q(z|x_0)
\right]
\\=:&\text{ELBO}_\theta(x_0).
\label{eq:elbo_general}
\end{align}
The quantity $\text{ELBO}_\theta(x_0)$ is the evidence lower bound. Maximizing the ELBO provides a tractable surrogate for maximizing $\log p_\theta(x_0)$, because it replaces the log of an integral with an expectation under $q(z|x_0)$.

Equivalently, using $p_\theta(z|x_0)=p_\theta(x_0,z)/p_\theta(x_0)$,
\begin{align}
\log p_\theta(x_0)
=
\text{ELBO}_\theta(x_0)
+
\mathrm{KL}\left(q(z|x_0)\|p_\theta(z|x_0)\right),
\label{eq:elbo_gap}
\end{align}
so the bound is tight if and only if $q(z|x_0)=p_\theta(z|x_0)$. Here the KL divergence is defined as $\mathrm{KL}(q\|p):=\mathbb{E}_{q}\left[\log\frac{q}{p}\right]$, corresponding to the $\gD^\text{L}_\theta$ in~\Cref{sec:prelim}
which is nonnegative and equals $0$ iff $q=p$. In diffusion models, a carefully chosen $q$ makes the ELBO decompose into simple local terms, yielding a practical learning objective while still targeting maximum likelihood.

\paragraph{Diffusion as structured variational inference.}
Diffusion models instantiate the generic latent variable as an entire noising
trajectory,
\begin{align}
z \equiv x_{1:T} := (x_1,\ldots,x_T),
\end{align}
and define a \emph{fixed} forward (variational) process
\begin{align}
q(z|x_0)
=
q(x_{1:T}|x_0)
=
\prod_{t=1}^{T} q(x_t|x_{t-1}),
\label{eq:q_forward}
\end{align}
where each $q(x_t|x_{t-1})$ is a simple corruption kernel.
The generative model uses a simple prior $p(x_T)$ and learned reverse transitions:
\begin{align}
p_\theta(x_0,z)
=\nonumber&
p_\theta(x_{0:T})\\
=&
p_\theta(x_0|x_1)
\prod_{t=2}^{T} p_\theta(x_{t-1}|x_t) p(x_T).
\label{eq:p_reverse}
\end{align}

Substituting~\Cref{eq:q_forward}--\Cref{eq:p_reverse} into~\Cref{eq:elbo_general}
yields a sum of tractable terms:
\begin{align}
&\text{ELBO}_\theta(x_0)\nonumber\\
=&
\mathbb{E}_{q(x_1|x_0)}\!\left[\log p_\theta(x_0|x_1)\right]
-
\mathrm{KL}\!\left(q(x_T|x_0)\|p(x_T)\right) \nonumber\\
&-
\sum_{t=2}^{T}
\mathbb{E}_{q(x_t|x_0)}\!\left[
\mathrm{KL}\!\left(
q(x_{t-1}|x_t,x_0)\|p_\theta(x_{t-1}|x_t)
\right)\right],
\label{eq:diffusion_elbo_kl}
\end{align}
where $q(x_{t-1}|x_t,x_0)$ is the exact posterior of the fixed forward process
(obtained from Bayes' rule). This expression highlights that maximizing
$\text{ELBO}_\theta(x_0)$ amounts to fitting each reverse kernel
$p_\theta(x_{t-1}|x_t)$ to the corresponding diffusion posterior under $q$.

\paragraph{EUBO (evidence upper bound).}
Using the same latent-variable setup, define the (nonnegative) importance weight
\begin{align}
w_\theta(z;x_0):=\frac{p_\theta(x_0,z)}{q(z|x_0)}.
\end{align}
Then the evidence satisfies $p_\theta(x_0)=\E_{q(z|x_0)}[w_\theta(z;x_0)]$.
For any $\beta\ge 1$, Jensen (applied to the convex map $u\mapsto u^\beta$) gives
\begin{align}
&\log p_\theta(x_0)
=\log \E_q[w_\theta]\nonumber\\
&\le\frac{1}{\beta}\log \E_{q(z|x_0)}\left[w_\theta(z;x_0)^\beta\right]
=:\text{EUBO}_{\theta,\beta}(x_0),
\label{eq:eubo_def}
\end{align}
so $\text{EUBO}_{\theta,\beta}(x_0)$ is an upper bound on the log-evidence (tight when $q(z|x_0)=p_\theta(z|x_0)$).

The looseness of this upper bound is naturally characterized by a Renyi divergence. In particular,
\begin{align*}
&\text{EUBO}_{\theta,\beta}(x_0)-\log p_\theta(x_0)\\
=&\frac{1}{\beta}\log \E_{q(z|x_0)}\left[\left(\frac{p_\theta(z|x_0)}{q(z|x_0)}\right)^\beta\right]\\
=&\frac{\beta-1}{\beta}D_\beta\left(p_\theta(z|x_0)\|q(z|x_0)\right).
\end{align*}
Here $D_\beta(p\|q):=\frac{1}{\beta-1}\log \E_{q}\left[\left(\frac{p}{q}\right)^\beta\right]$ is the Renyi divergence of order $\beta$, corresponding to the $\gD^\text{U}_\theta$ in~\Cref{sec:prelim}. This quantity is nonnegative and equals $0$ iff $p_\theta(z|x_0)=q(z|x_0)$, so the EUBO is tight exactly when the variational posterior matches the true posterior.

In diffusion, the latent variable is the noising trajectory $z\equiv x_{1:T}$ with fixed forward process $q(x_{1:T}|x_0)$, and $p_\theta(x_0,z)=p_\theta(x_{0:T})$ is defined by the learned reverse transitions.
Thus,
\begin{align}
\text{EUBO}_{\theta,\beta}(x_0)
=\frac{1}{\beta}\log \E_{q(x_{1:T}|x_0)}\left[
\left(\frac{p_\theta(x_{0:T})}{q(x_{1:T}|x_0)}\right)^\beta
\right],
\end{align}
which is typically much harder to optimize than the ELBO because it involves a log-moment over entire forward paths and does not decompose into a simple sum of per-timestep KL terms.
Directly maximizing the likelihood $\log p_\theta(x_0)=\log\int p_\theta(x_0,z)\mathrm{d}z$ is generally intractable because it requires marginalizing over high-dimensional latents $z$ (in diffusion, entire trajectories).
Likewise, $\text{EUBO}_{\theta,\beta}(x_0)$ is typically intractable since it involves a log-moment $\log \E_q[w_\theta^\beta]$ over the same latent space, which does not admit a simple additive decomposition.
In contrast, the diffusion ELBO is tractable because it moves the logarithm inside an expectation under the fixed forward process $q$, yielding terms that can be estimated by sampling $z\sim q(\cdot|x_0)$ (often reducing to per-timestep losses).

SPG~\citep{wang2025spg} obtains a practical EUBO surrogate for masked dLLMs by exploiting the special absorbing-mask forward process.
Let $x_{1:n}$ be a token sequence, where $x_i$ is the token at position $i$, and let $m$ denote the mask token.
In this paragraph, $t$ indexes the diffusion noising step, while $i$ indexes the sequence position.
We write $z_t=(z_{t,1},\dots,z_{t,n})$ for the corrupted sequence at diffusion time $t$, so $z_{t,i}$ is the $i$-th token of the corrupted sequence at that time.
The absorbing-mask forward process uses a monotone schedule $\alpha_t$ with, in the discrete-time convention, $\alpha_1=1$ and $\alpha_T=0$.
At time $t$, each coordinate remains clean with probability $\alpha_t$ and is masked with probability $1-\alpha_t$:
\begin{align*}
    q(z_{t,i}=x_i|x_i)=\alpha_t,\qquad
    q(z_{t,i}=m|x_i)=1-\alpha_t .
\end{align*}
Thus each coordinate is either unchanged or masked, and once a coordinate becomes $m$ it stays masked under the forward process.
For this process, the only nontrivial reverse-model term for coordinate $i$ occurs when $z_{t+1,i}=m$ and the model predicts the clean token $x_i$ from the corrupted sequence $z_{t+1}$.
The factor $(\alpha_t-\alpha_{t+1})/(1-\alpha_{t+1})$ below is the posterior probability that coordinate $i$ was still clean at time $t$ given that it is masked at time $t+1$.

If we directly instantiate the Renyi EUBO in~\Cref{eq:eubo_def} for this masked diffusion model, we obtain a path-level log-moment:
\begin{align}
\mathrm{EUBO}^{\mathrm{path}}_{\theta,\beta}(x_{1:n})
=
\frac{1}{\beta}\log
\E_{z_{1:T}\sim q(\cdot|x_{1:n})}
\left[
\prod_{t=1}^{T-1}\prod_{i=1}^n
\left(
\frac{p_\theta(z_{t,i}|z_{t+1})}
{q(z_{t,i}|z_{t+1},x_{1:n})}
\right)^\beta
\right].
\label{eq:masked_path_eubo}
\end{align}
This is the literal, unapproximated EUBO for the masked diffusion model, but it is not in the same form as the usual masked-diffusion ELBO training loss: it contains a product over all diffusion times and token positions inside a single outer logarithm.
By contrast, the masked-diffusion ELBO moves the logarithm inside the expectation and reduces to weighted masked-token prediction terms,
\begin{align*}
\mathcal{L}_{\mathrm{ELBO}}(x_{1:n};\theta)
=
\sum_{i=1}^n
\E_{t,z_t}
\left[
w(t)\mathbbm{1}\{z_{t,i}=m\}\log\pi_\theta(x_i|z_t)
\right],
\end{align*}
up to schedule-dependent constants.
SPG's approximation sits between these two objects: it keeps an upper-bound structure derived from the path-level EUBO, but uses the absorbing-mask structure to obtain a per-coordinate log-moment that can be estimated with masked-token prediction probabilities.
Applying the Renyi variational bound and decomposing over coordinates yields the discrete-time upper-bound surrogate
\begin{align}
\mathcal{L}^{\mathrm{SPG}}_{\mathrm{EUBO}}(x_{1:n};\theta)
&=
\frac{1}{\beta}\sum_{i=1}^n
\log\Bigg(
\sum_{t=1}^{T-1}
\E_{z_{t+1}\sim q(\cdot|x_{1:n})}
\Bigg[
\frac{\alpha_t-\alpha_{t+1}}{1-\alpha_{t+1}}
\mathbbm{1}\{z_{t+1,i}=m\}
\nonumber\\
&\hspace{14em}
\pi_\theta(x_i|z_{t+1})^\beta
\Bigg]\Bigg)
+C(T),
\label{eq:spg_dllm_eubo}
\end{align}
where $C(T)$ is independent of $\theta$ and therefore does not affect policy-gradient updates.
In the continuous-time implementation, this is written in the equivalent form
\begin{align}
\widetilde{\mathcal{L}}^{\mathrm{SPG}}_{\mathrm{EUBO}}(x_{1:n};\theta)
=
\frac{1}{\beta}\sum_{i=1}^n
\log
\E_{t,z_t}
\left[
w(t)\mathbbm{1}\{z_{t,i}=m\}\pi_\theta(x_i|z_t)^\beta
\right],
\label{eq:spg_dllm_eubo_cont}
\end{align}
with $w(t)$ collecting the continuous-time analogue of the schedule-dependent masking weight above, including any density correction for how $t$ is sampled.
The logarithm remains outside the expectation, so single-sample Monte Carlo estimates of this quantity are generally biased; nevertheless, its gradient can be estimated by differentiating the sampled log-moment.
This construction is specific to the categorical absorbing-mask structure of dLLMs, which is why the SPG EUBO approximation does not directly provide a general-purpose EUBO estimator for arbitrary diffusion or flow policies.

\section{Theoretical Analysis of DiPOD}\label{app:theory}

In this section, we formally analyze the convergence of DiPOD.
In particular, we formalize and prove~\Cref{thm:informal_dipod}.

\subsection{Setup and Assumptions}
Throughout this section, we work in the fully observable MDP setting described in~\Cref{sec:rl_diffusion}, so observations are states.
To keep the notation aligned with the main body, we write $o$ for this observation/state variable rather than switching to $s$.
Let $d_{\pi_\theta}$ be the on-policy observation distribution induced by $\pi_\theta$ in the same sense used by the policy-gradient theorem (for example, the discounted occupancy measure in the episodic case or the stationary distribution in the continuing case), and define
\begin{align*}
    \rho_\theta(o,a):=d_{\pi_\theta}(o)\pi_\theta(a| o).
\end{align*}
For any reference distribution $\mu$ over observation-action pairs, define the self-distillation objective
\begin{align*}
    F_\mu(\theta):=\E_\mu\Mp{\elbo_\theta(a| o)}
\end{align*}
and the expected ELBO and EUBO discrepancies
\begin{align*}
    D_\mu^{\mathrm{L}}(\theta)
    &:=
    \E_{(o,a)\sim\mu}\Mp{\gD_\theta^{\mathrm{L}}(o,a)}
    =
    \E_\mu\Mp{\log\pi_\theta(a| o)-\elbo_\theta(a| o)},\\
    D_\mu^{\mathrm{U}}(\theta)
    &:=
    \E_{(o,a)\sim\mu}\Mp{\gD_\theta^{\mathrm{U}}(o,a)}
    =
    \E_\mu\Mp{\eubo_\theta(a| o)-\log\pi_\theta(a| o)}.
\end{align*}
For a gradient estimator $g_\theta(o,a)$, write its on-policy expected update as
\begin{align*}
    G_g(\theta)
    :=
    \E_{(o,a)\sim\rho_\theta}\Mp{g_\theta(o,a)}.
\end{align*}

We are going to build our analysis on the following standard assumptions.

\begin{assumption}[Smooth Policy-Gradient Dynamics]\label{ass:smooth_pg}
The policy-gradient theorem holds for $\gJ$ on the region of parameter space visited by the algorithm, and $\gJ$ is $L_\gJ$-smooth:
\begin{align*}
    \gJ(\theta')
    \ge
    \gJ(\theta)+\langle\nabla\gJ(\theta),\theta'-\theta\rangle
    -\frac{L_\gJ}{2}\|\theta'-\theta\|^2 .
\end{align*}
\end{assumption}

\begin{assumption}[Lipschitzness of Gradient Estimator]\label{ass:quant_adequate}
The estimator $g_\theta$ is adequate in the sense of Definition~\ref{def:adequate}.
Moreover, there exists an adequacy coefficient $c_g<\infty$ such that whenever the current on-policy ELBO discrepancy is at most $\varepsilon$,
\begin{align*}
    D_{\rho_\theta}^{\mathrm{L}}(\theta)\le\varepsilon
    \quad\Longrightarrow\quad
    \left\|G_g(\theta)-\nabla\gJ(\theta)\right\|
    \le
    c_g\sqrt{\varepsilon}.
\end{align*}
This is the quantitative form of adequacy used in the theorem: near a tight on-policy ELBO, the expected adequate-gradient update is Lipschitz-close to the true policy gradient.
\end{assumption}

\begin{assumption}[Realizability]\label{ass:realizable_tight_distill}
For every reference policy $\pi_\mathrm{ref}$ considered by the algorithm, with reference distribution $\mu(o,a)=d_\mathrm{ref}(o)\pi_\mathrm{ref}(a| o)$, there exists a parameter $\theta^\dagger$ such that $\pi_{\theta^\dagger}(\cdot| o)=\pi_\mathrm{ref}(\cdot| o)$ for $d_\mathrm{ref}$-almost every $o$ and $\gD_{\theta^\dagger}^{\mathrm{L}}(o,a)=0$ for $\mu$-almost every $(o,a)$.
Equivalently, the policy class can represent the reference policy with a tight ELBO on the reference distribution.
\end{assumption}

\begin{assumption}[On-policy Self-Distillation Oracle]\label{ass:on_policy_oracle}
At each self-distillation step, DiPOD chooses the current on-policy reference distribution $\mu=\rho_{\theta_\mathrm{ref}}$.
The self-distillation oracle returns a parameter $\bar{\theta}$, and hence a policy $\pi_{\bar{\theta}}$, whose on-reference ELBO value is within $\varepsilon$ of the best achievable value:
\begin{align*}
    F_\mu(\bar{\theta})
    \ge
    \sup_\theta F_\mu(\theta)-\varepsilon.
\end{align*}
Equivalently, the returned policy/model is $\varepsilon$-suboptimal for the on-reference ELBO maximization problem.
This suboptimality may come from either policy mismatch with the reference distribution or from a nonzero ELBO discrepancy.
\end{assumption}

\begin{remark}[Reference versus returned-policy distributions]\label{rem:reference_vs_returned}
The oracle assumption controls ELBO optimality under the reference distribution used for self-distillation.
An approximate oracle can still return a policy whose own occupancy distribution differs from the reference occupancy distribution.
For this reason, the performance theorem below keeps track of both the reference-distribution discrepancy guaranteed by self-distillation and the returned policy's own on-policy discrepancy, which is the quantity needed to control the next policy-gradient update.
\end{remark}

Before proceeding to the performance guarantee, we first show that both FPO and SPG satisfy Assumption~\ref{ass:quant_adequate} under standard bounded-advantage and smoothness conditions.
The same self-bounding argument can also be extended to other adequate-gradient estimators whose bias is controlled by smooth nonnegative variational discrepancies.

\begin{proposition}[FPO and SPG instantiate Lipschitzness]\label{prop:adequate_estimators}
Assume $|A^{\pi_\theta}(o,a)|\le A_{\max}$.
\begin{enumerate}
    \item \textbf{FPO.}
    Let
    \begin{align*}
        g_\theta^\mathrm{FPO}(o,a)
        =
        A^{\pi_\theta}(o,a)\nabla_\theta\elbo_\theta(a| o).
    \end{align*}
    If $\gD_\theta^{\mathrm{L}}(o,a)$ is nonnegative and $L_{\mathrm{L}}$-smooth in $\theta$ for each $(o,a)$, then
    \begin{align*}
        \left\|G_{g^\mathrm{FPO}}(\theta)-\nabla\gJ(\theta)\right\|
        \le
        A_{\max}\sqrt{2L_{\mathrm{L}}D_{\rho_\theta}^{\mathrm{L}}(\theta)}.
    \end{align*}
    Thus FPO satisfies~\Cref{ass:quant_adequate} with $c_g=A_{\max}\sqrt{2L_{\mathrm{L}}}$.

    \item \textbf{SPG.}
    Let
    \begin{align*}
        g_\theta^\mathrm{SPG}(o,a)
        =
        \mathbbm{1}_{A>0}A\nabla_\theta\elbo_\theta(a| o)
        +
        \mathbbm{1}_{A<0}A\nabla_\theta\eubo_\theta(a| o),
    \end{align*}
    where $A$ abbreviates $A^{\pi_\theta}(o,a)$.
    If $\gD_\theta^{\mathrm{L}}$ and $\gD_\theta^{\mathrm{U}}$ are nonnegative and respectively $L_{\mathrm{L}}$- and $L_{\mathrm{U}}$-smooth, then
    \begin{align*}
        \left\|G_{g^\mathrm{SPG}}(\theta)-\nabla\gJ(\theta)\right\|
        \le
        A_{\max}\sqrt{2L_{\mathrm{L}}D_{\rho_\theta}^{\mathrm{L}}(\theta)}
        +
        A_{\max}\sqrt{2L_{\mathrm{U}}D_{\rho_\theta}^{\mathrm{U}}(\theta)}.
    \end{align*}
    Consequently, if the EUBO discrepancy is also tight in the self-distilled region, e.g. $D_{\rho_\theta}^{\mathrm{U}}(\theta)\le C_{\mathrm{U}}D_{\rho_\theta}^{\mathrm{L}}(\theta)$, then SPG satisfies~\Cref{ass:quant_adequate} with $c_g=A_{\max}(\sqrt{2L_{\mathrm{L}}}+\sqrt{2L_{\mathrm{U}}C_{\mathrm{U}}})$.
\end{enumerate}
\end{proposition}

\begin{proof}
We use the self-bounding property of nonnegative smooth functions.
If $f(\theta)\ge 0$ and $f$ is $L$-smooth, then
\begin{align}
    \|\nabla f(\theta)\|^2\le 2Lf(\theta).
    \label{eq:self_bounding_gap}
\end{align}
Indeed, smoothness at $\theta'=\theta-\frac{1}{L}\nabla f(\theta)$ gives
\begin{align*}
    0
    \le
    f(\theta')
    \le
    f(\theta)-\frac{1}{2L}\|\nabla f(\theta)\|^2.
\end{align*}

For FPO, the policy-gradient theorem gives
\begin{align*}
    \nabla\gJ(\theta)
    =
    \E_{\rho_\theta}
    \Mp{
    A^{\pi_\theta}(o,a)\nabla_\theta\log\pi_\theta(a| o)
    }.
\end{align*}
Using $\elbo_\theta=\log\pi_\theta-\gD_\theta^{\mathrm{L}}$,
\begin{align*}
    G_{g^\mathrm{FPO}}(\theta)-\nabla\gJ(\theta)
    =
    -
    \E_{\rho_\theta}
    \Mp{A^{\pi_\theta}(o,a)\nabla_\theta\gD_\theta^{\mathrm{L}}(o,a)}.
\end{align*}
Therefore,
\begin{align*}
    \left\|G_{g^\mathrm{FPO}}(\theta)-\nabla\gJ(\theta)\right\|
    &\le
    A_{\max}\E_{\rho_\theta}\Mp{\left\|\nabla_\theta\gD_\theta^{\mathrm{L}}(o,a)\right\|}
    \\
    &\le
    A_{\max}
    \sqrt{
    \E_{\rho_\theta}
    \Mp{\left\|\nabla_\theta\gD_\theta^{\mathrm{L}}(o,a)\right\|^2}
    }
    \\
    &\le
    A_{\max}\sqrt{2L_{\mathrm{L}}D_{\rho_\theta}^{\mathrm{L}}(\theta)},
\end{align*}
where the last step uses~\Cref{eq:self_bounding_gap}.

For SPG, use $\eubo_\theta=\log\pi_\theta+\gD_\theta^{\mathrm{U}}$ in addition to the ELBO decomposition:
\begin{align*}
    G_{g^\mathrm{SPG}}(\theta)-\nabla\gJ(\theta)
    =
    -
    \E_{\rho_\theta}
    \Mp{\mathbbm{1}_{A>0}A\nabla_\theta\gD_\theta^{\mathrm{L}}(o,a)}
    +
    \E_{\rho_\theta}
    \Mp{\mathbbm{1}_{A<0}A\nabla_\theta\gD_\theta^{\mathrm{U}}(o,a)}.
\end{align*}
The triangle inequality, bounded advantages, Cauchy-Schwarz, and~\Cref{eq:self_bounding_gap} give
\begin{align*}
    \left\|G_{g^\mathrm{SPG}}(\theta)-\nabla\gJ(\theta)\right\|
    &\le
    A_{\max}\sqrt{2L_{\mathrm{L}}D_{\rho_\theta}^{\mathrm{L}}(\theta)}
    +
    A_{\max}\sqrt{2L_{\mathrm{U}}D_{\rho_\theta}^{\mathrm{U}}(\theta)}.
\end{align*}
The final statement follows by substituting $D_{\rho_\theta}^{\mathrm{U}}(\theta)\le C_{\mathrm{U}}D_{\rho_\theta}^{\mathrm{L}}(\theta)$.
\end{proof}

\subsection{Performance Guarantee}
\begin{lemma}[Guarantee for Self-Distillation Steps]\label{lem:oracle_tightness}
Let $\mu(o,a)=d_\mathrm{ref}(o)\pi_\mathrm{ref}(a| o)$ be a reference state-action distribution satisfying~\Cref{ass:realizable_tight_distill}.
If an ELBO oracle returns $\bar{\theta}$ satisfying
\begin{align*}
    F_\mu(\bar{\theta})
    \ge
    \sup_\theta F_\mu(\theta)-\varepsilon,
\end{align*}
then
\begin{align}
    \E_{o\sim d_\mathrm{ref}}
    \Mp{
    \mathrm{KL}\!\left(
    \pi_\mathrm{ref}(\cdot| o)\,\|\,\pi_{\bar{\theta}}(\cdot| o)
    \right)}
    +
    D_\mu^{\mathrm{L}}(\bar{\theta})
    \le
    \varepsilon .
    \label{eq:oracle_kl_gap}
\end{align}
\end{lemma}

\begin{proof}
We proceed in three steps.

\paragraph{Step 1: upper bound the best possible self-distillation objective.}
For any candidate parameter $\theta$, the ELBO decomposition in~\Cref{eq:elbo} gives
\begin{align*}
    F_\mu(\theta)
    &=
    \E_\mu\Mp{\log\pi_\theta(a| o)}
    -
    D_\mu^{\mathrm{L}}(\theta).
\end{align*}
Since $D_\mu^{\mathrm{L}}(\theta)\ge 0$,
\begin{align*}
    F_\mu(\theta)
    \le
    \E_{o\sim d_\mathrm{ref},\,a\sim\pi_\mathrm{ref}(\cdot| o)}
    \Mp{\log\pi_\theta(a| o)}.
\end{align*}
Rewriting the right-hand side relative to $\pi_\mathrm{ref}$,
\begin{align*}
    &\E_{o\sim d_\mathrm{ref},\,a\sim\pi_\mathrm{ref}(\cdot| o)}
    \Mp{\log\pi_\theta(a| o)}
    \\
    &=
    \E_\mu\Mp{\log\pi_\mathrm{ref}(a| o)}
    -
    \E_{o\sim d_\mathrm{ref}}
    \Mp{
    \mathrm{KL}\!\left(
    \pi_\mathrm{ref}(\cdot| o)\,\|\,\pi_\theta(\cdot| o)
    \right)}
    \\
    &\le
    \E_\mu\Mp{\log\pi_\mathrm{ref}(a| o)}.
\end{align*}
Thus,
\begin{align}
    F_\mu(\theta)
    \le
    \E_\mu\Mp{\log\pi_\mathrm{ref}(a| o)} .
    \label{eq:self_distill_upper}
\end{align}

\paragraph{Step 2: show that the upper bound is attainable.}
By realizability, $F_\mu(\theta^\dagger)=\E_\mu[\log\pi_\mathrm{ref}(a| o)]$.
Together with~\Cref{eq:self_distill_upper}, this implies
\begin{align*}
    \sup_\theta F_\mu(\theta)
    =
    \E_\mu\Mp{\log\pi_\mathrm{ref}(a| o)}.
\end{align*}

\paragraph{Step 3: use oracle optimality.}
The $\varepsilon$-optimality of $\bar{\theta}$ gives
\begin{align*}
    F_\mu(\bar{\theta})
    \ge
    \E_\mu\Mp{\log\pi_\mathrm{ref}(a| o)}
    -
    \varepsilon .
\end{align*}
Expanding $F_\mu(\bar{\theta})$ as in Step 1 gives
\begin{align*}
    F_\mu(\bar{\theta})
    =
    \E_\mu\Mp{\log\pi_\mathrm{ref}(a| o)}
    -
    \E_{o\sim d_\mathrm{ref}}
    \Mp{
    \mathrm{KL}\!\left(
    \pi_\mathrm{ref}(\cdot| o)\,\|\,\pi_{\bar{\theta}}(\cdot| o)
    \right)}
    -
    D_\mu^{\mathrm{L}}(\bar{\theta}).
\end{align*}
Combining the previous two displays yields~\Cref{eq:oracle_kl_gap}.
\end{proof}

\begin{theorem}[DiPOD improves return under on-policy tightness]\label{thm:on_policy_dipod}
Suppose~\Cref{ass:smooth_pg,ass:quant_adequate,ass:realizable_tight_distill,ass:on_policy_oracle} hold.
Consider one idealized DiPOD cycle with current parameter $\theta_k$.
The self-distillation oracle uses $\rho_{\theta_k}$ as reference and returns $\bar{\theta}_k$.
Then~\Cref{lem:oracle_tightness} gives the reference-distribution decomposition
\begin{align}
    \E_{o\sim d_{\pi_{\theta_k}}}
    \Mp{
    \mathrm{KL}\!\left(
    \pi_{\theta_k}(\cdot| o)\,\|\,\pi_{\bar{\theta}_k}(\cdot| o)
    \right)}
    +
    D_{\rho_{\theta_k}}^{\mathrm{L}}(\bar{\theta}_k)
    \le
    \varepsilon.
    \label{eq:oracle_decomp_cycle}
\end{align}
As emphasized in~\Cref{rem:reference_vs_returned}, this is a guarantee on the reference distribution $\rho_{\theta_k}$.
Now let $\bar{\varepsilon}_k:=D_{\rho_{\bar{\theta}_k}}^{\mathrm{L}}(\bar{\theta}_k)$ be the discrepancy on the returned policy's own on-policy distribution, and perform the policy update
\begin{align*}
    \theta_{k+1}
    =
    \bar{\theta}_k+\eta G_g(\bar{\theta}_k).
\end{align*}
If $0<\eta\le 1/L_\gJ$, then
\begin{align}
    \gJ(\theta_{k+1})
    \ge
    \gJ(\bar{\theta}_k)
    +
    \frac{\eta}{2}
    \Mp{
    \left\|\nabla\gJ(\bar{\theta}_k)\right\|^2
    -
    c_g^2\bar{\varepsilon}_k
    }.
    \label{eq:dipod_improvement_bound}
\end{align}
Therefore, whenever $\|\nabla\gJ(\bar{\theta}_k)\|>c_g\sqrt{\bar{\varepsilon}_k}$, the policy update strictly improves the true return relative to the returned policy $\bar{\theta}_k$.
If the oracle is exact ($\varepsilon=0$), or more generally returns a policy-preserving model so that $\rho_{\bar{\theta}_k}=\rho_{\theta_k}$, then $\bar{\varepsilon}_k\le\varepsilon$ and $\gJ(\bar{\theta}_k)=\gJ(\theta_k)$; in this case the full DiPOD cycle strictly improves $\gJ$ whenever $\|\nabla\gJ(\bar{\theta}_k)\|>c_g\sqrt{\varepsilon}$.

Moreover, suppose the iterates converge to a stable post-distillation policy $\theta_\star$ for which the final on-policy self-distillation step is $\varepsilon$-optimal and returns $\theta_\star$.
Here stable means that the subsequent idealized update
\begin{align*}
    \theta_\star^+
    =
    \theta_\star+\eta G_g(\theta_\star)
\end{align*}
with $0<\eta\le 1/L_\gJ$ does not strictly improve $\gJ$.
Then
\begin{align*}
    \left\|\nabla\gJ(\theta_\star)\right\|
    \le
    c_g\sqrt{\varepsilon},
    \qquad
    D_{\rho_{\theta_\star}}^{\mathrm{L}}(\theta_\star)\le\varepsilon.
\end{align*}
Thus the final policy is approximately stationary for the true return and has a tight ELBO on its own final on-policy distribution.
\end{theorem}

\begin{proof}
The oracle decomposition~\Cref{eq:oracle_decomp_cycle} follows directly by applying~\Cref{lem:oracle_tightness} with $\mu=\rho_{\theta_k}$.

By definition of $\bar{\varepsilon}_k$ and the Lipschitzness of the estimator,
\begin{align*}
    \left\|
    G_g(\bar{\theta}_k)-\nabla\gJ(\bar{\theta}_k)
    \right\|
    \le
    c_g\sqrt{\bar{\varepsilon}_k}.
\end{align*}
Let
\begin{align*}
    v:=\nabla\gJ(\bar{\theta}_k),
    \qquad
    e:=G_g(\bar{\theta}_k)-v.
\end{align*}
Then $\|e\|\le c_g\sqrt{\bar{\varepsilon}_k}$.
Using $L_\gJ$-smoothness with
$\theta_{k+1}=\bar{\theta}_k+\eta(v+e)$,
\begin{align*}
    \gJ(\theta_{k+1})
    &\ge
    \gJ(\bar{\theta}_k)
    +
    \eta\langle v,v+e\rangle
    -
    \frac{L_\gJ\eta^2}{2}\|v+e\|^2 .
\end{align*}
When $\eta\le 1/L_\gJ$,
\begin{align*}
    \gJ(\theta_{k+1})
    &\ge
    \gJ(\bar{\theta}_k)
    +
    \eta\langle v,v+e\rangle
    -
    \frac{\eta}{2}\|v+e\|^2
    \\
    &=
    \gJ(\bar{\theta}_k)
    +
    \frac{\eta}{2}\Mp{\|v\|^2-\|e\|^2}
    \\
    &\ge
    \gJ(\bar{\theta}_k)
    +
    \frac{\eta}{2}
    \Mp{
    \left\|\nabla\gJ(\bar{\theta}_k)\right\|^2
    -
    c_g^2\bar{\varepsilon}_k
    }.
\end{align*}
This proves~\Cref{eq:dipod_improvement_bound}.
Strict improvement follows whenever $\|\nabla\gJ(\bar{\theta}_k)\|>c_g\sqrt{\bar{\varepsilon}_k}$.
If the oracle is policy preserving, then $\rho_{\bar{\theta}_k}=\rho_{\theta_k}$, so~\Cref{eq:oracle_decomp_cycle} implies $\bar{\varepsilon}_k=D_{\rho_{\theta_k}}^{\mathrm{L}}(\bar{\theta}_k)\le\varepsilon$, and policy preservation also gives $\gJ(\bar{\theta}_k)=\gJ(\theta_k)$.
When $\varepsilon=0$, the oracle decomposition gives zero expected KL from $\pi_{\theta_k}$ to $\pi_{\bar{\theta}_k}$ on $d_{\pi_{\theta_k}}$ and zero reference-distribution ELBO discrepancy; under the same MDP dynamics this is the exact policy-preserving case.

For the final claim, apply the oracle decomposition to the final self-distillation step with reference $\rho_{\theta_\star}$ and returned parameter $\theta_\star$.
The KL term is zero, so $D_{\rho_{\theta_\star}}^{\mathrm{L}}(\theta_\star)\le\varepsilon$.
If $\|\nabla\gJ(\theta_\star)\|>c_g\sqrt{\varepsilon}$, the improvement bound applied to $\theta_\star^+$ would give a strict increase in $\gJ$, contradicting stability.
Thus $\|\nabla\gJ(\theta_\star)\|\le c_g\sqrt{\varepsilon}$.
\end{proof}

\section{Additional Details for Two-Token Post-Training}\label{app:toy_details}

We use the same toy experiment as the one in Appendix C.3 of SPG~\citep{wang2025spg}. In the toy experiment, the discrete diffusion model generates a distribution on two discrete tokens $x=(x_1,x_2)$. Both $x_1$ and $x_2$ take values from $\gV=\{\mathrm{A},\mathrm{B}\}$. In the generation process, $x_1$ and $x_2$ can also be a special mask token $\mathrm{M}$. The generation starts from $x=\text{MM}$, and tokens are decoded according to a uniformly random order. The model is parameterized by six real numbers:
\begin{align*}
    a=\text{logit}(\pi(x_1=\mathrm{A}|x=\mathrm{MA})),
    b=\text{logit}(\pi(x_1=\mathrm{A}|x=\mathrm{MM})),
    c=\text{logit}(\pi(x_2=\mathrm{A}|x=\mathrm{AM})),\\
    d=\text{logit}(\pi(x_2=\mathrm{A}|x=\mathrm{MM})),
    e=\text{logit}(\pi(x_1=\mathrm{A}|x=\mathrm{MB})),
    f=\text{logit}(\pi(x_2=\mathrm{A}|x=\mathrm{BM})),
\end{align*}
where the logit function is defined as $\displaystyle\text{logit}(p)=\ln\frac{p}{1-p}$. One can explicitly calculate $\log\pi,\text{ELBO}$ and $\mathcal{D}^\text{L}$ using these parameters to monitor the drift. As an explicit example, the log-likelihood and ELBO for AA can be calculated as
\begin{align*}
    \log\pi(\text{AA})&=\log\Sp{\frac12\text{S}(b)\text{S}(c)+\frac12\text{S}(a)\text{S}(d)},\\
    \text{ELBO}(\text{AA})&=\frac12(\log\text{S}(a)+\log\text{S}(b)+\log\text{S}(c)+\log\text{S}(d)).
\end{align*}
Here $\displaystyle\text{S}(x)=\frac{1}{1+e^{-x}}$ is the sigmoid function, the inverse function of the logit function, which recovers probabilities from logits. The parameters are all initialized to be $0.5$ so that ELBO equals log-likelihood for all outputs, resulting in a pretrained diffusion model. We set the reward function to be
\begin{align*}
    r(\mathrm{AA})=0.8,r(\mathrm{AB})=1,r(\mathrm{BA})=0.7,r(\mathrm{BB})=1.
\end{align*}
The reward function is chosen for clarity of exposition and is without loss of generality.

\paragraph{Algorithm implementation.} We implement FPO by updating the model parameters according to~\Cref{eq:fpo}, and SPG by updating the model parameters according to~\Cref{eq:spg}. In the SPG implementation, a surrogate for EUBO (the $\mathcal{L}_\text{EUBO}$ in~\cite{wang2025spg}) is used to tackle the intractability issue, and we follow their implementation in our experiment. In this toy setting, we directly calculate the policy gradient without invoking Monte Carlo samples for gradient estimation. We implement~\Cref{alg:implementation} with FPO and SPG gradient estimators for comparison as well. We set the learning rate to be $0.1$, the beta parameter in EUBO to be $1.5$, $\beta$ in DiPOD to be $0.2$, and run these algorithms for $1500$ policy-gradient steps.

\section{Additional Experiments for Diffusion Language Models}
\subsection{Experiments with More Sequence Lengths}

We additionally report results at generation sequence lengths $128$, $256$, and $512$ for SPG and SPG+DiPOD. We use the same setup as in~\Cref{sec:exp_language}, and set the DiPOD coefficient to $\beta=0.05$ for all settings except Sudoku with sequence lengths $128$ and $512$, where we use $\beta=0.02$. We summarize all tasks and sequence lengths in a single table. For GSM8K, MATH500, and Countdown, we copy the LLaDA-8B-Instruct and d1 baselines from SPG~\citep{wang2025spg}. For Sudoku, we do \emph{not} use the 3-shot numbers reported in SPG; instead, we report the zero-shot LLaDA-8B-Instruct and d1 results from d1~\citep{zhao2025d1}, and the wd1 baseline from wd1~\citep{tang2025wd1}. The wd1 paper reports zero-shot results only at sequence lengths $256$ and $512$; for GSM8K, MATH500, and Countdown at sequence length $128$, we therefore use the reproduced wd1 numbers in SPG~\citep{wang2025spg}, while the Sudoku $128$ entry remains unavailable because the SPG Sudoku setting is 3-shot.

\begin{table*}[t]
\centering
\scriptsize
\setlength{\tabcolsep}{3pt}
\begin{tabular}{lcccccccccccc}
\toprule
& \multicolumn{3}{c}{GSM8K} & \multicolumn{3}{c}{MATH500} & \multicolumn{3}{c}{Countdown} & \multicolumn{3}{c}{Sudoku} \\
\cmidrule(lr){2-4}\cmidrule(lr){5-7}\cmidrule(lr){8-10}\cmidrule(lr){11-13}
Method & 128 & 256 & 512 & 128 & 256 & 512 & 128 & 256 & 512 & 128 & 256 & 512 \\
\midrule
LLaDA-8B-Instruct & 69.5 & 77.2 & 79.8 & 28.2 & 32.4 & 34.6 & 18.8 & 16.8 & 16.8 & 11.7 & 6.7 & 5.5 \\
d1 & 72.2 & 80.6 & 81.3 & 31.4 & 36.0 & 39.4 & 30.9 & 30.9 & 34.4 & 22.1 & 16.7 & 9.5 \\
wd1 & 74.6 & 80.8 & 82.3 & 31.0 & 34.4 & 39.0 & 48.8 & 51.2 & 46.1 & N/A & 76.4 & 62.8 \\
SPG & 81.05 & 84.23 & 83.40 & 33.40 & 37.80 & 41.80 & 68.80 & 51.95 & 70.30 & 44.40 & 25.12 & 40.62 \\
SPG+DiPOD & 81.65 & 84.91 & 85.29 & 36.00 & 40.00 & 41.40 & 44.53 & 80.08 & 55.47 & 25.02 & 97.56 & 36.42 \\
\bottomrule
\end{tabular}
\caption{Results across generation sequence lengths. For GSM8K, MATH500, and Countdown, the LLaDA-8B-Instruct and d1 baselines are copied from SPG. For Sudoku, we report zero-shot baselines from d1 and wd1 rather than the 3-shot numbers in SPG. For wd1, sequence length $128$ is taken from SPG's reproduced wd1 row for GSM8K, MATH500, and Countdown, since the original wd1 paper only reports zero-shot results at lengths $256$ and $512$; the Sudoku $128$ entry is therefore unavailable.}
\label{tab:spg_more_lengths}
\end{table*}

Compared with SPG, SPG+DiPOD gives consistent gains on GSM8K and MATH500 across sequence lengths. On Countdown, the best performance shifts to sequence length $256$, where SPG+DiPOD substantially outperforms SPG, while SPG remains stronger at $128$ and $512$. Sudoku exhibits a discrete training behavior, with runs tending to settle around roughly $25\%$, $40\%$, or near $100\%$ accuracy. Although SPG+DiPOD does not show a clear improvement at sequence lengths $128$ and $512$, it is the only setting that reaches the near-$100\%$ regime in the zero-shot setting.

\subsection{Ablation on the DiPOD Coefficient $\beta$}

We ablate the DiPOD coefficient $\beta$ in Equation~(7). Due to computational constraints, we conduct this study only on Countdown. This choice is also motivated by task characteristics: on GSM8K and MATH500, the gains are relatively small in our main results, while on Sudoku, single-run outcomes can exhibit discrete phase-transition-like behavior, which may lead to misleading conclusions in a limited ablation. Therefore, Countdown provides the cleanest setting for isolating the effect of $\beta$. Consistent with the fairness protocol used in the main experiments, we keep the random seed fixed across all ablation runs.

\begin{table}[h]
\centering
\small
\caption{Ablation of the DiPOD coefficient $\beta$ on Countdown. We report final accuracy under the same setup as~\Cref{sec:exp_language}.}
\label{tab:beta_ablation_countdown}
\begin{tabular}{lc}
\toprule
$\beta$ & Performance \\
\midrule
0        & 51.95 \\
0.01     & 53.91 \\
0.03     & 55.47 \\
0.05     & 80.08 \\
0.055    & 76.56 \\
0.06     & 71.09 \\
0.07     & 52.73 \\
0.10     & 55.86 \\
\bottomrule
\end{tabular}
\end{table}

\Cref{tab:beta_ablation_countdown} shows that the choice of $\beta$ has a clear effect on performance. A moderate regularization strength works best, with $\beta=0.05$ achieving the strongest result. When $\beta$ is too small, the additional ELBO tightening is not strong enough to sufficiently control the drift during policy optimization; when $\beta$ is too large, the optimization appears to overemphasize the regularization term, which can weaken the policy-improving update. Overall, these results support the use of a moderate on-policy ELBO regularization strength and validate our default choice of $\beta=0.05$ in the main language-model experiments.

\subsection{Experiment on Sudoku in the Three-Shot Setting}
\begin{figure}[h]
    \centering
    \includegraphics[width=0.5\linewidth]{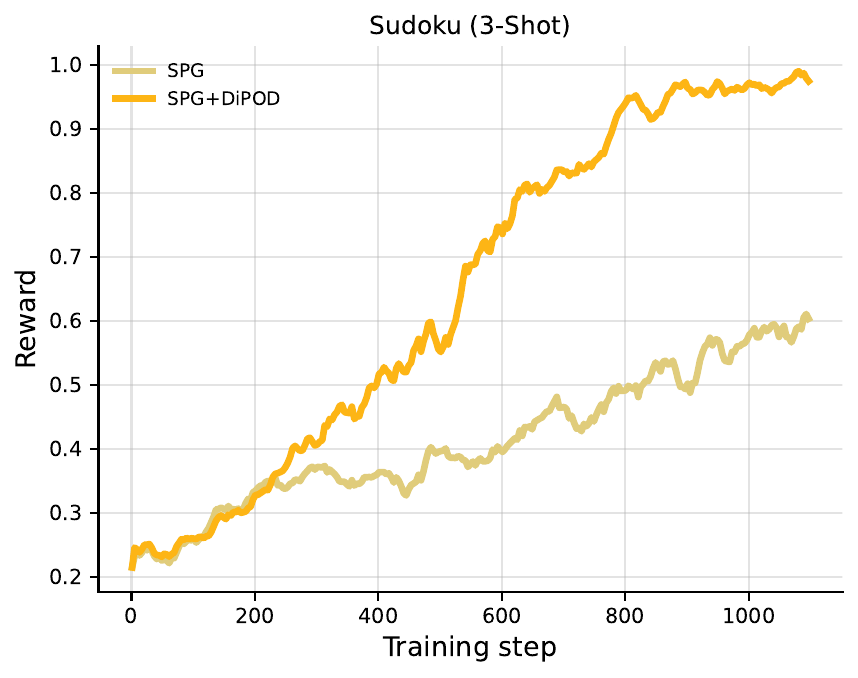}
    \caption{Reward dynamics of SPG and SPG+DiPOD in the 3-shot setting.}
    \label{fig:3shot}
\end{figure}
In the three-shot setting, SPG~\citep{wang2025spg} has already saturated the performance. Here we show the reward dynamics of SPG+DiPOD compared to SPG in~\Cref{fig:3shot}, with $\beta=0.05$. We can see that DiPOD converges faster to the optimal performance than SPG.

\section{Additional Experiments for Motion Tracking}
\subsection{Motion Tracking Experiment}

We provide additional details for the motion-tracking experiment discussed in the main text. We follow the motion-tracking setup of FPO++~\citep{yi2026flow}: a diffusion policy controls the Unitree G1 humanoid to track reference motions from LAFAN~\citep{harvey2020robust}. All downstream FPO++ hyperparameters follow the original motion-tracking setup. The only added component is an initial self-distillation stage before normal FPO++ policy-gradient training. We use this single-stage version because policy-preserving self-distillation during high-dimensional motion-control training is itself a nontrivial algorithmic component; designing a fully interleaved self-distillation schedule for this setting is left to future work.

\paragraph{Self-distillation procedure.}
The self-distillation stage instantiates the self-distillation step in the original interleaved DiPOD procedure~\Cref{alg:interleave}. At the beginning of training, we clone the randomly initialized actor into a frozen teacher. We then collect rollouts using teacher actions: the frozen teacher samples actions, and we store the resulting observation and teacher action pairs. The student actor is trained on this growing dataset by maximizing the empirical average of the evidence lower bound. This phase is not reward learning: the critic, advantages, and surrogate policy-gradient loss are used only after self-distillation is complete. Thus the initial distillation step preserves the initial policy distribution while tightening the gap between ELBO and log-likelihood before the subsequent policy-gradient stage.

\begin{algorithm}[h]
\caption{Initial self-distillation for motion tracking}
\label{alg:motion_self_distillation}
\begin{algorithmic}[1]
\REQUIRE initial actor $\pi_\theta$, vectorized environments, distillation iterations $K$
\STATE Clone the actor into a frozen teacher $\bar{\pi}$
\STATE Initialize replay buffer $\gD \leftarrow \emptyset$
\FOR{$k=1,\dots,K$}
    \STATE Collect teacher rollouts in the live environments and append sampled observation-action pairs $(o,a)$ to $\gD$
    \STATE Sample a minibatch $\{(o_i,a_i)\}_{i=1}^B$ from $\gD$
    \STATE Update only the student actor by maximizing $\frac{1}{B}\sum_{i=1}^B \elbo_\theta(a_i|o_i)$
\ENDFOR
\STATE Reset the optimizer state and continue with the FPO++  policy-gradient stage
\end{algorithmic}
\end{algorithm}

In our runs, self-distillation uses $100$ iterations, $8$ rollout steps per iteration, minibatch size $16384$, AdamW with learning rate $3\cdot10^{-4}$ and weight decay $10^{-4}$, and actor gradient clipping at $1.0$. The frozen teacher uses $64$ sampling steps. The tracking tasks use $4096$ parallel environments.

\paragraph{Computational cost.}
We run the experiments on a single NVIDIA L40S GPU. For \emph{dance\_1\_subject\_1}, the self-distillation stage takes $95.5$s in total, or about $0.96$s per self-distillation iteration. A policy-gradient iteration takes about $15$s, so the $1000$-iteration training runs shown in the figures take roughly $4.2$ hours. Thus, the initial self-distillation stage adds only about $1.5$ minutes of computation, which is negligible relative to the full training process.

\paragraph{Results.}
We report the full set of motion-tracking curves in~\Cref{fig:motion_tracking_reward_all,fig:motion_tracking_length_all}. Across all six LAFAN tracking tasks, DiPOD improves sample efficiency over the reproduced FPO++ baseline: the reward curves rise faster and the episode-length curves indicate more sustained successful tracking. The consistency across \emph{dance\_1\_subject\_1}, \emph{dance\_1\_subject\_2}, \emph{fight\_1\_subject\_2}, \emph{jumps\_1\_subject\_1}, \emph{run\_1\_subject\_2}, and \emph{walk\_1\_subject\_1} suggests that the initial self-distillation step provides a general benefit rather than a task-specific gain.

\begin{figure}[p]
    \centering
    \begin{subfigure}{0.31\textwidth}
        \centering
        \includegraphics[width=\linewidth]{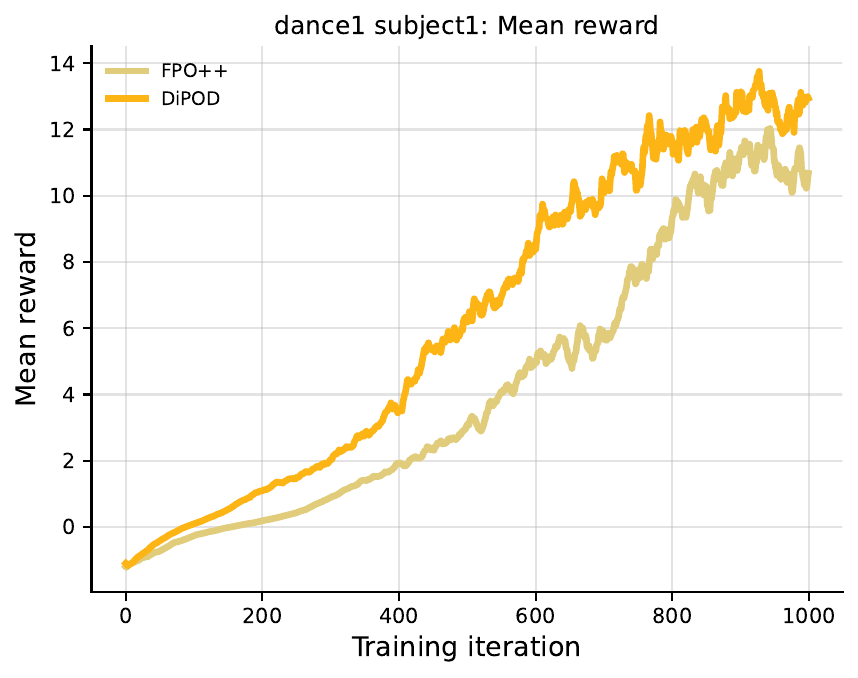}
        \caption{\emph{dance\_1\_subject\_1}}
    \end{subfigure}
    \hfill
    \begin{subfigure}{0.31\textwidth}
        \centering
        \includegraphics[width=\linewidth]{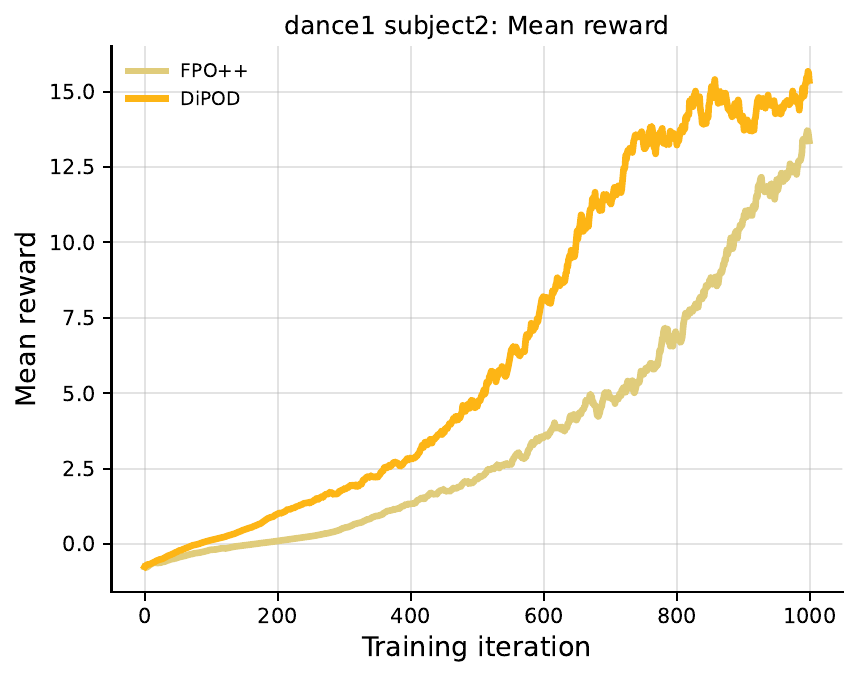}
        \caption{\emph{dance\_1\_subject\_2}}
    \end{subfigure}
    \hfill
    \begin{subfigure}{0.31\textwidth}
        \centering
        \includegraphics[width=\linewidth]{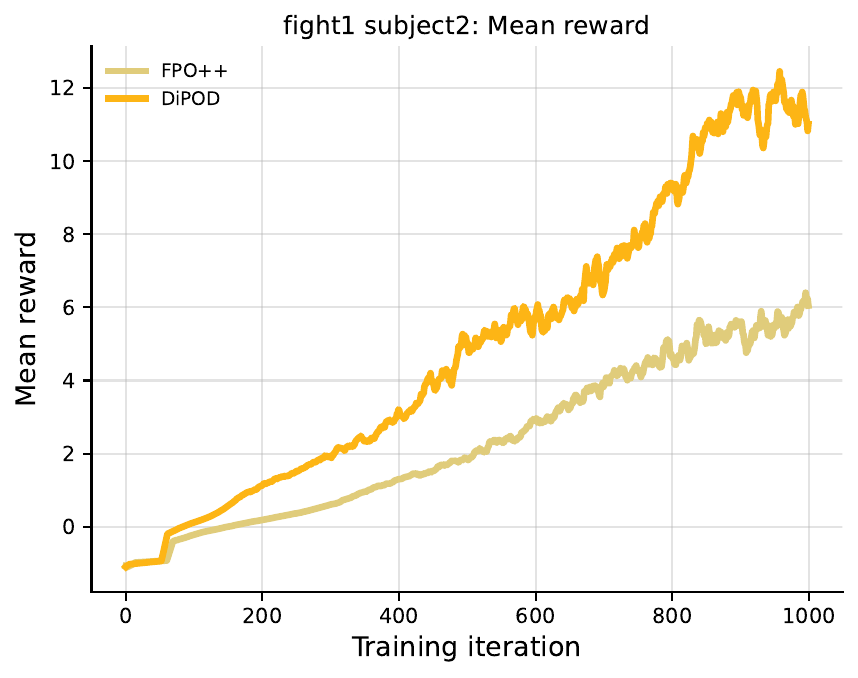}
        \caption{\emph{fight\_1\_subject\_2}}
    \end{subfigure}

    \begin{subfigure}{0.31\textwidth}
        \centering
        \includegraphics[width=\linewidth]{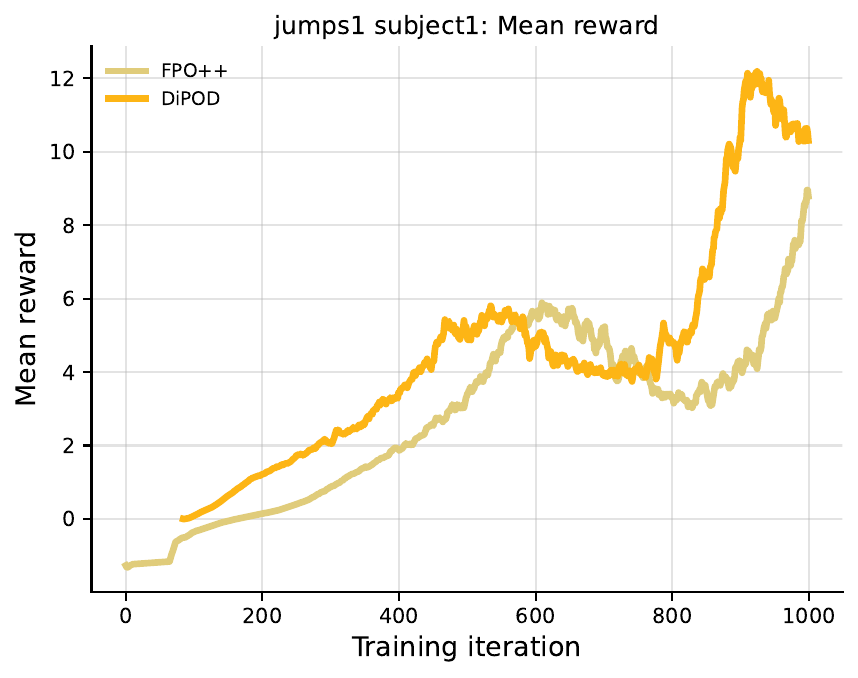}
        \caption{\emph{jumps\_1\_subject\_1}}
    \end{subfigure}
    \hfill
    \begin{subfigure}{0.31\textwidth}
        \centering
        \includegraphics[width=\linewidth]{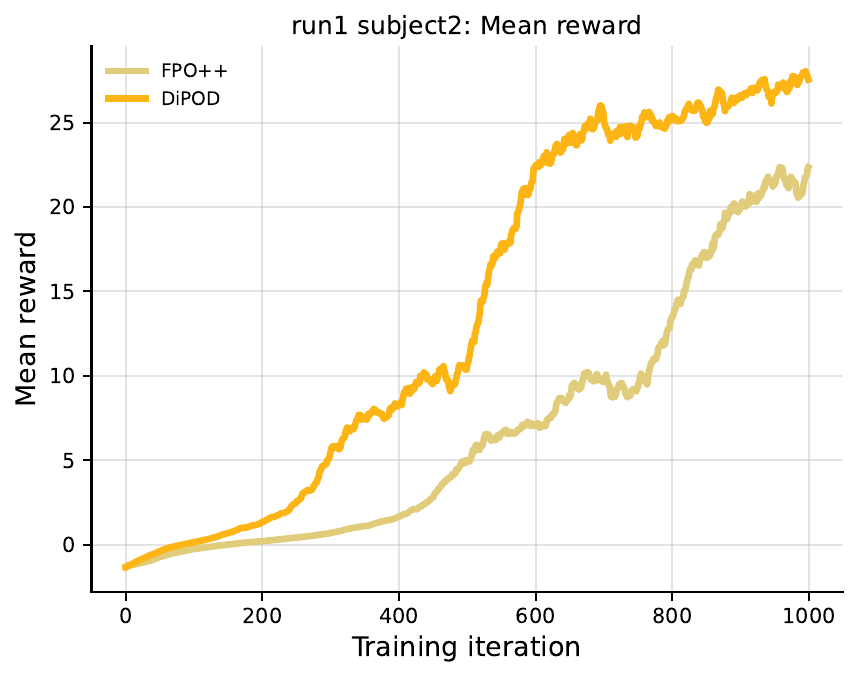}
        \caption{\emph{run\_1\_subject\_2}}
    \end{subfigure}
    \hfill
    \begin{subfigure}{0.31\textwidth}
        \centering
        \includegraphics[width=\linewidth]{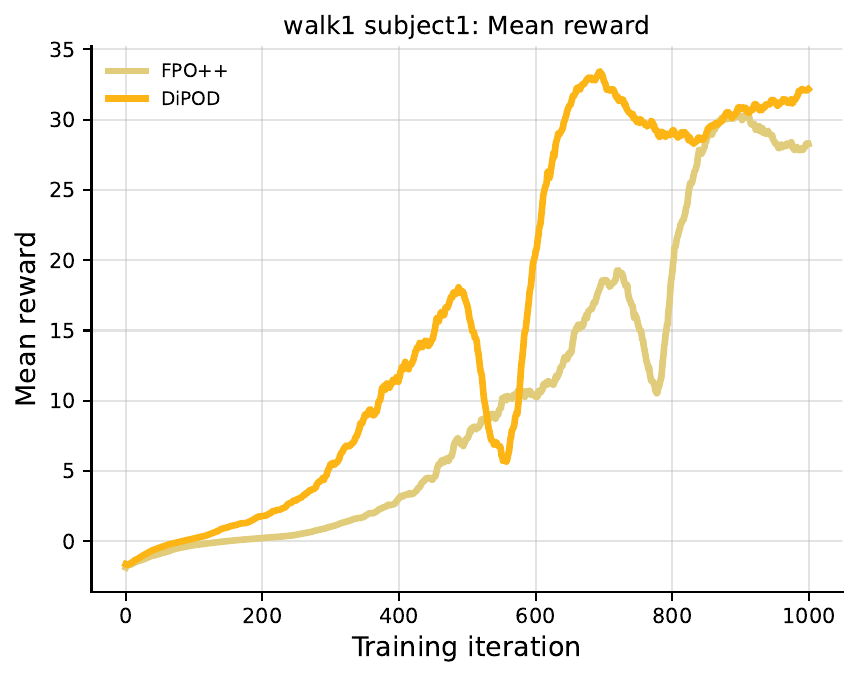}
        \caption{\emph{walk\_1\_subject\_1}}
    \end{subfigure}
    \caption{Mean reward curves for all six LAFAN motion-tracking tasks on the G1 humanoid.}
    \label{fig:motion_tracking_reward_all}
\end{figure}

\begin{figure}[p]
    \centering
    \begin{subfigure}{0.31\textwidth}
        \centering
        \includegraphics[width=\linewidth]{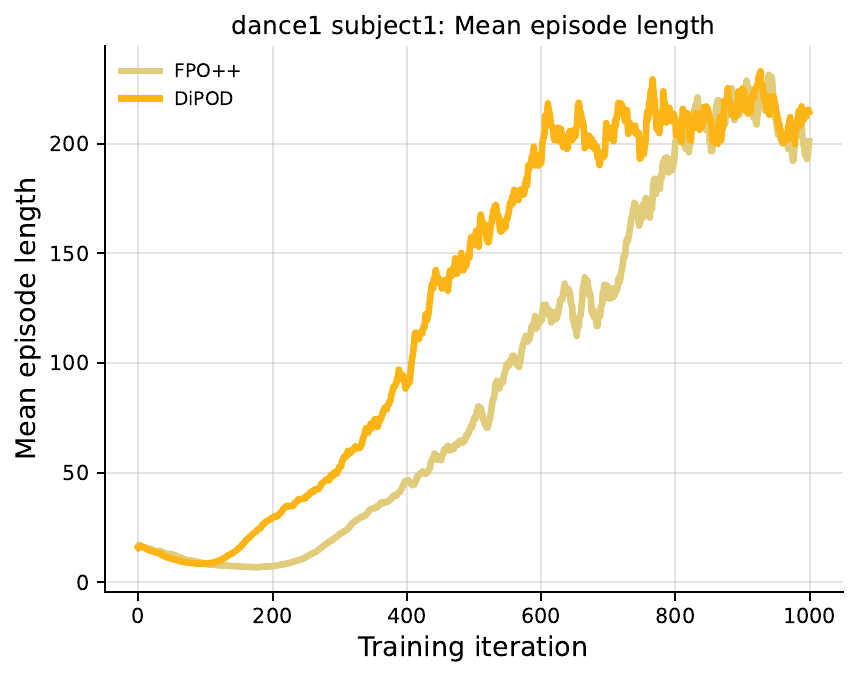}
        \caption{\emph{dance\_1\_subject\_1}}
    \end{subfigure}
    \hfill
    \begin{subfigure}{0.31\textwidth}
        \centering
        \includegraphics[width=\linewidth]{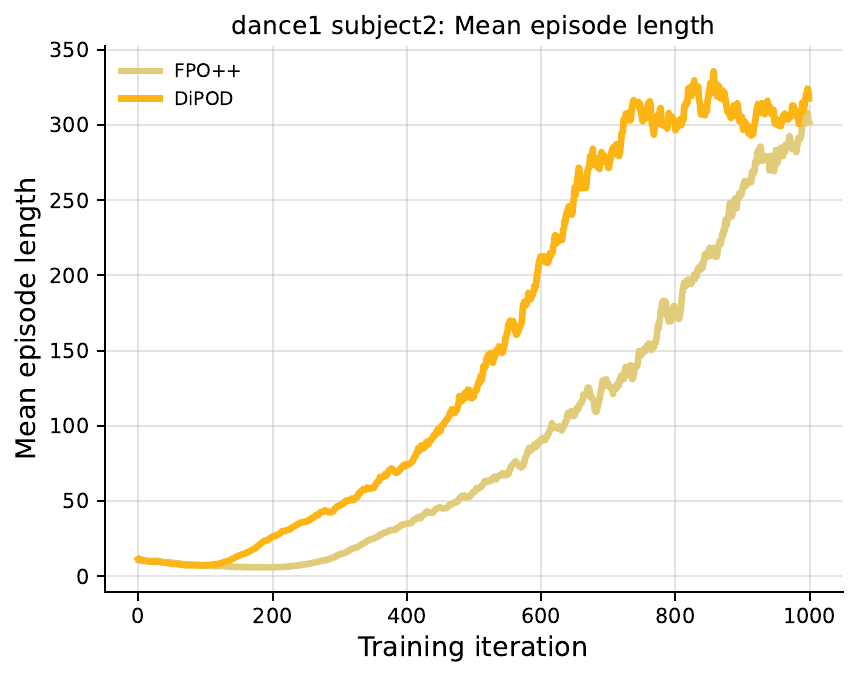}
        \caption{\emph{dance\_1\_subject\_2}}
    \end{subfigure}
    \hfill
    \begin{subfigure}{0.31\textwidth}
        \centering
        \includegraphics[width=\linewidth]{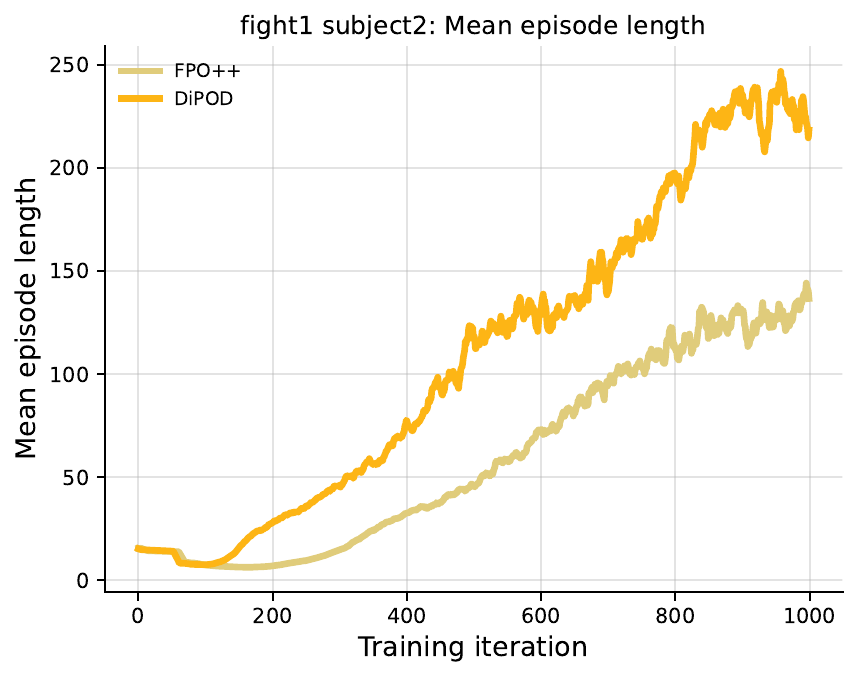}
        \caption{\emph{fight\_1\_subject\_2}}
    \end{subfigure}

    \begin{subfigure}{0.31\textwidth}
        \centering
        \includegraphics[width=\linewidth]{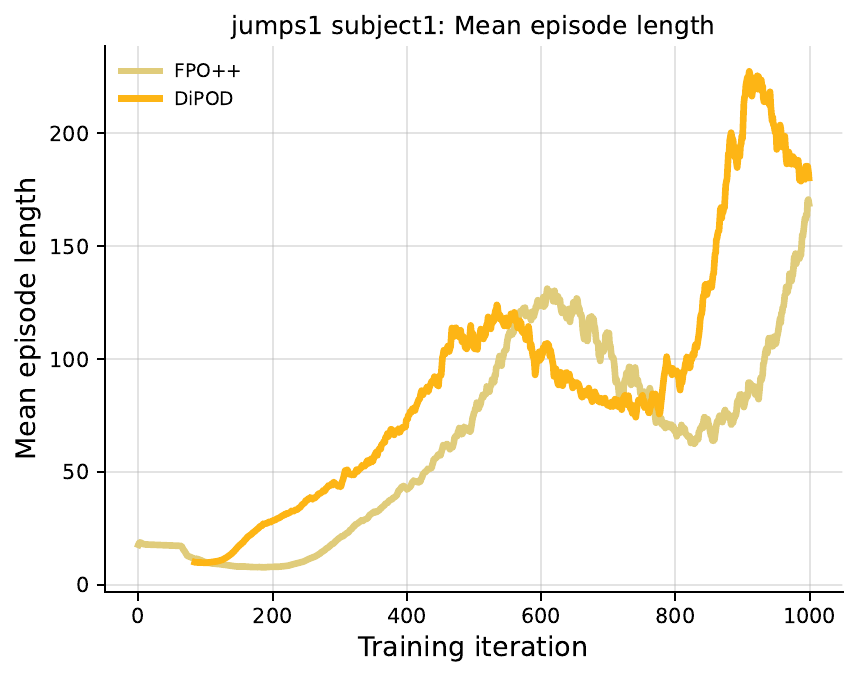}
        \caption{\emph{jumps\_1\_subject\_1}}
    \end{subfigure}
    \hfill
    \begin{subfigure}{0.31\textwidth}
        \centering
        \includegraphics[width=\linewidth]{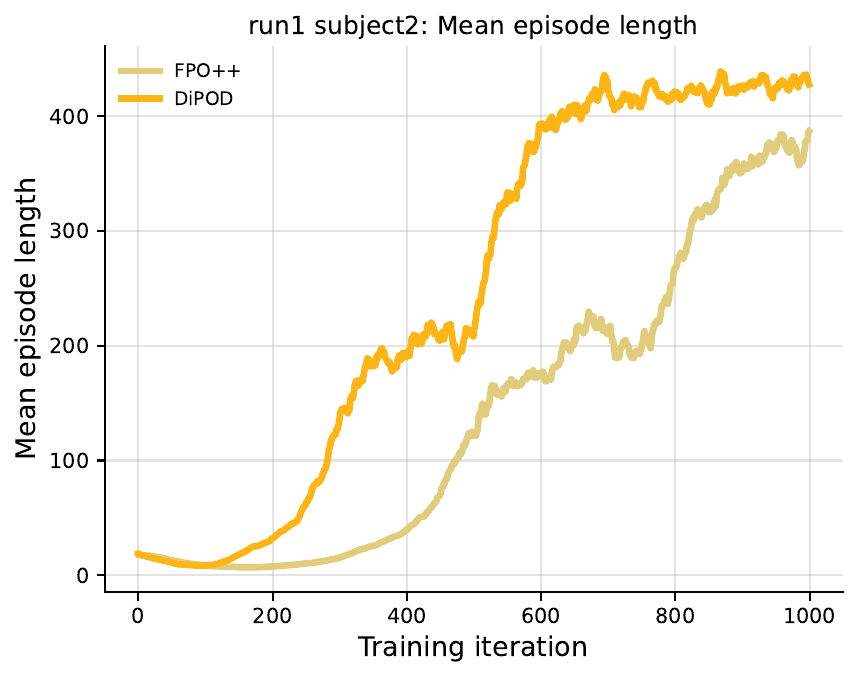}
        \caption{\emph{run\_1\_subject\_2}}
    \end{subfigure}
    \hfill
    \begin{subfigure}{0.31\textwidth}
        \centering
        \includegraphics[width=\linewidth]{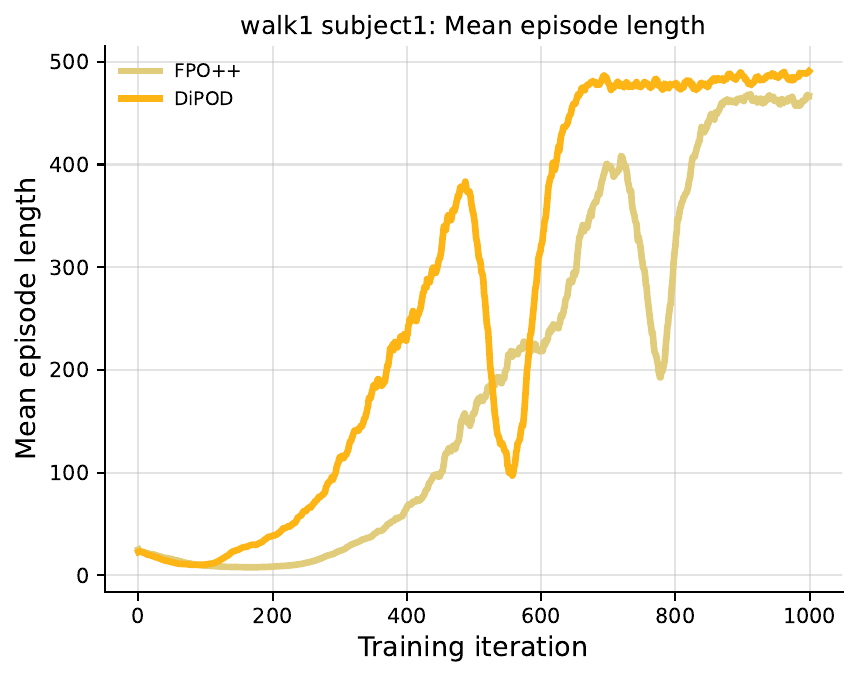}
        \caption{\emph{walk\_1\_subject\_1}}
    \end{subfigure}
    \caption{Mean episode length curves for all six LAFAN motion-tracking tasks on the G1 humanoid.}
    \label{fig:motion_tracking_length_all}
\end{figure}

\end{document}